\documentclass[letterpaper, 10 pt, conference]{ieeeconf}
\IEEEoverridecommandlockouts                              
\PassOptionsToPackage{table,xcdraw}{xcolor}
\usepackage{tikz}

\usepackage{comment}
\usepackage{notoccite}           
\usepackage[english]{babel}
\makeatletter
\newcommand{\l@en}{\l@english}
\makeatother

\usepackage[linesnumbered,ruled,vlined]{algorithm2e}

\usepackage{amssymb}
\usepackage{mathtools}           
\usepackage{bm}                  
\usepackage{cancel}              

\usepackage{tabularx}
\usepackage{booktabs}
\usepackage{array}
\usepackage{arydshln}            
\usepackage{makecell}
\usepackage{tikz}                
\usepackage{pgfplots}            
\usepackage{pgfplotstable}      
\usepgfplotslibrary{groupplots}
\pgfplotsset{compat=newest}
\usetikzlibrary{
	positioning,
	arrows.meta,
	calc,
	shapes.symbols,
	patterns,
	shapes
}
\pgfkeys{ 
  /pgf/number format/.cd, set decimal separator={.},
  set thousands separator = {} 
}

\makeatletter
\let\NAT@parse\undefined
\makeatother
\usepackage[colorlinks=true, allcolors=blue, hypertexnames=false]{hyperref}
\usepackage{bookmark}

\usepackage[capitalize]{cleveref}            

\usepackage[export]{adjustbox} 

\newcommand{\vect}[1]{\bm{#1}}
\newcommand{\mat}[1]{\bm{#1}}
\newcommand{\T}{^T}              

\title{
From Passive Monitoring to Active Defence: Resilient Control of Manipulators Under Cyberattacks
}

\author{Gabriele Gualandi and Alessandro V. Papadopoulos
\thanks{This work was supported by the Swedish Research Council (VR) with the PSI project No. \#2020-05094, and by the Knowledge Foundation (KKS).}
\thanks{G. Gualandi and A.V. Papadopoulos are with the Department of Computer Science and Engineering,
        M{\"a}lardalen University, V{\"a}ster{\aa}s, Sweden
        {\tt\small gabriele.gualandi@mdu.se}}%
}

\begin{document}

\definecolor{light_grey}{rgb}{0.9,0.9,0.9}
\definecolor{dark_green}{rgb}{0.35,0.7,0.2}
\definecolor{good_pink}{RGB}{255,20,147}

\definecolor{color0}{rgb}{0.125,0.125,0.875}

\definecolor{c1}{RGB}{55,65,220} 
\definecolor{c2}{RGB}{62,118,70} 
\definecolor{c3}{RGB}{228,26,28} 

\definecolor{c4}{RGB}{152,78,163} 
\definecolor{c5}{RGB}{255,127,0} 
\definecolor{c6}{RGB}{146,146,27} 
\definecolor{c7}{RGB}{0,0,0} 

\newcounter{MaxActuationvalue}
\setcounter{MaxActuationvalue}{5}

\newcounter{Nvalue}
\setcounter{Nvalue}{500}
\newcounter{ARLvalue}
\setcounter{ARLvalue}{5000}

\newcounter{SimulationDurationvalue}
\setcounter{SimulationDurationvalue}{2000}
\newcounter{AttackDurationvalue}
\setcounter{AttackDurationvalue}{1200}
\newcounter{AttackStartIterationvalue}
\setcounter{AttackStartIterationvalue}{800}


\newtheorem{lemma}{Lemma}
\newtheorem{proposition}{Proposition}
\newtheorem{corollary}{Corollary}
\newtheorem{theorem}{Theorem}
\newtheorem{assumption}{Assumption}
\newtheorem{remark}{Remark}
\newtheorem{definition}{Definition}


\newcommand{\RealNumber}{\mathbb{R}}
\newcommand{\TimeStep}{k}
\newcommand{\NumJoint}{n}
\newcommand{\TaskSpaceSize}{l}
\newcommand{\PositionSize}{v}
\newcommand{\OrientationSize}{g}

\newcommand{\NumInput}{m}
\newcommand{\ActuationSymbol}{u}
\newcommand{\ActuationSaturated}{\boldsymbol{\ActuationSymbol}}
\newcommand{\AttackedActuation}{\tilde{\ActuationSaturated}}
\newcommand{\ActuationUnsaturated}{\cancel{\ActuationSaturated}}

\newcommand{\Sat}[1]{\mathrm{sat}(#1)}
\newcommand{\ActuationSaturationMin}{\ActuationSymbol_\mathrm{min}}
\newcommand{\ActuationSaturationMax}{\ActuationSymbol_\mathrm{max}}

\newcommand{\Position}{\boldsymbol{p}}
\newcommand{\Orientation}{\boldsymbol{o}}

\newcommand{\TargetPosition}{\bar{\Position}}
\newcommand{\PositionEst}{\hat{\Position}}
\newcommand{\TargetRotationMatrix}{\bar{\mat{\RotationMatrix}}}
\newcommand{\TargetAngular}{\bar{\Orientation}}
\newcommand{\TargetAngularVelocity}{\bar{\bm{\omega}}}
\newcommand{\EstimatedAngularVelocity}{\hat{\bm{\omega}}}
\newcommand{\TargetPositionAttack}{\bar{\Position}^{\text{A}}}
\newcommand{\TargetVelocityAttack}{\dot{\bar{\Position}}^{\text{A}}}
\newcommand{\PlannedAccAttack}{\ddot{\bar{\Position}}^{\text{A}}}
\newcommand{\TargetAccAttack}{{\ddot{\Position}}^{\text{A}}}

\newcommand{\EstimatedPositionProj}{\projSymbol{\Position}}
\newcommand{\EstimatedRotationMatrix}{\hat{\mat{\RotationMatrix}}}

\newcommand{\ErrorVector}{\boldsymbol{e}}
\newcommand{\ErrorPosition}[1]{\ErrorVector_{\Position,#1}}
\newcommand{\ErrorPositionDerivative}[1]{\dot{\ErrorVector}_{\Position,#1}}
\newcommand{\ErrorOrientation}[1]{\ErrorVector_{\Orientation,#1}}
\newcommand{\ErrorAngular}[1]{\dot{\ErrorVector}_{\Orientation,#1}}

\newcommand{\GainProportionalPosition}{\bm{K}_{pp}}
\newcommand{\GainDerivativePosition}{\bm{K}_{dp}}
\newcommand{\GainProportionalOrientation}{\bm{K}_{po}}
\newcommand{\GainDerivativeOrientation}{\bm{K}_{do}}

\newcommand{\PIDActuationTask}{\ddot{\bm{x}}^{\text{c}}}
\newcommand{\ActuationPD}[1]{\ActuationSaturated_{\mathrm{PD}, #1}}

\newcommand{\AngleAxisAxis}{\hat{\boldsymbol{r}}}
\newcommand{\AngleAxisAngle}{\hat{\theta}}

\newcommand{\TaskSpaceState}{\boldsymbol{\xi}}
\newcommand{\UserSetpointTask}{\bar{\TaskSpaceState}}

\newcommand{\NumState}{h}
\newcommand{\StateVector}{\boldsymbol{x}}
\newcommand{\TrueState}{\StateVector}
\newcommand{\EstimatedState}{\hat{\StateVector}}


\newcommand{\NumOutput}{p}
\newcommand{\TrueOutput}{\boldsymbol{y}}
\newcommand{\AttackedOutput}{\tilde{\TrueOutput}}
\newcommand{\EstimatedOutput}{\hat{\TrueOutput}}

\newcommand{\Residual}{\boldsymbol{r}}

\newcommand{\ProcessNoise}{\boldsymbol{w}}
\newcommand{\ProcessNoiseCovar}{\Upsilon}
\newcommand{\MeasurementNoise}{\boldsymbol{v}}
\newcommand{\MeasurementNoiseCovar}{\Gamma}

\newcommand{\KalmanGain}{\mat{L}}

\newcommand{\desprob}{\psi}
\newcommand{\MahalanobisDistance}{z}
\newcommand{\MahalanobisWindow}{w}
\newcommand{\WindowSize}{N}
\newcommand{\ThresholdVal}{\tau}
\newcommand{\AlarmSignal}{s_{\mathrm{alarm}}}
\newcommand{\ARL}{\mathcal{A}}
\newcommand{\InverseGammaFunction}[2]{\mathcal{P}^{-1}(#1,#2)}


\newcommand{\MeasurementAttack}{\bm{a}}
\newcommand{\AttackDuration}{T}
\newcommand{\SynchDuration}{T^{\text{SYN}}}


\newcommand{\DirectionVector}{\boldsymbol{d}}
\newcommand{\EstimatedAttackDirection}{\DirectionVector}
\newcommand{\ManipulabilityMatrix}{\mathbf{M}}
\newcommand{\ManipulabilityMeasure}{\eta}
\newcommand{\CostFunction}{\mathcal{C}}
\newcommand{\GradientDescentGain}{\nu}
\newcommand{\GradientDescentGainMax}{{\GradientDescentGain}_\mathrm{max}}

\newcommand{\WeightedPIMatrix}{\mathbf{W}_\TimeStep}
\newcommand{\WeightedJacobian}{\mat{J}_{\WeightedPIMatrix}}
\newcommand{\BalanceScalar}{\alpha}
\newcommand{\ShiftFactor}{\gamma}
\newcommand{\ShiftFactorTolerance}{\epsilon}

\newcommand{\ActuationNULL}[1]{\ActuationSaturated_{\mathrm{low}, #1}}

\newcommand{\RotationMatrix}{R}
\newcommand{\GeometricJacobian}{\mathbf{J}}

\newcommand{\joint}{\vect{q}}
\newcommand{\djoint}{\dot{\joint}}

\newcommand{\jointEst}{\hat{\joint}}
\newcommand{\djointEst}{\hat{\djoint}}

\newcommand{\projSymbol}{\tilde}
\newcommand{\ActuationProjectedState}{\projSymbol{\StateVector}}
\newcommand{\ActuationProjectedJointConfiguration}{\projSymbol{\joint}}
\newcommand{\ActuationProjectedJointVelocity}[1]{\projSymbol{\djoint}_{#1}}
\newcommand{\linearizedPdd}{\mathcal{G}}
\newcommand{\simul}{\mathcal{Z}}
\newcommand{\simulLin}{\bm{Z}}

\newcommand{\GainAttackProportional}{\boldsymbol{K}^{\text{A}}_{\text{p}}}
\newcommand{\GainAttackDerivative}{\boldsymbol{K}^{\text{A}}_{\text{d}}}
\newcommand{\PredictedPosition}{\Position^{\text{SIM}}}
\newcommand{\PredictedVelocity}{\dot{\Position}^{\text{SIM}}}
\newcommand{\PredictedAcceleration}{\ddot{\Position}^{\text{SIM}}}
\newcommand{\lambdaEffort}{\lambda_{\text{e}}}
\newcommand{\baselineIn}{\bm{c}}
\newcommand{\optimAmat}{\bm{O}}
\newcommand{\optLambda}{\zeta}

\newcommand{\pinExp}{f}

\newcommand{\heightOne}{3.0cm} 
\newcommand{\heightTwo}{2.4cm} 

\newcommand{\mean}[1]{\operatorname{mean}\bigl(#1\bigr)}

\newcommand{\devmax}[2]{\operatorname{devmax}\bigl( \{ #1 \}, \{ #2 \} \bigr) }
\newcommand{\devRMS}[2]{\operatorname{devRMS}\bigl( \{ #1 \}, \{ #2 \} \bigr) }



\providecommand{\ICRAImgsPath}{imgs}

\pgfplotstableread[col sep=comma]{\ICRAImgsPath/NoDefence_NoADS_YesAttackConstants.csv}\ConstNoDefeNoADSYesAtt
\pgfplotstableread[col sep=comma]{\ICRAImgsPath/NoDefence_YesAttackConstants.csv}\ConstNoDefeYesAtt
\pgfplotstableread[col sep=comma]{\ICRAImgsPath/YesDefence_YesAttackConstants.csv}\ConstYesDefYesAtt

\pgfplotstableread[col sep=comma]{\ICRAImgsPath/NoDefence_NoADS_YesAttackStringResults.csv}%
    \ExpNoDefeNoADSYesAttStr
\pgfplotstableread[col sep=comma]{\ICRAImgsPath/NoDefence_YesAttackStringResults.csv}%
    \ExpNoDefeYesAttStr
\pgfplotstableread[col sep=comma]{\ICRAImgsPath/YesDefence_YesAttackStringResults.csv}%
    \ExpYesDefYesAttStr

\pgfplotstableread[col sep=comma]{\ICRAImgsPath/NoDefence_NoADS_YesAttackNumericResults.csv}%
    \ExpNoDefeNoADSYesAttNum
\pgfplotstableread[col sep=comma]{\ICRAImgsPath/NoDefence_YesAttackNumericResults.csv}%
    \ExpNoDefeYesAttNum
\pgfplotstableread[col sep=comma]{\ICRAImgsPath/YesDefence_YesAttackNumericResults.csv}%
    \ExpYesDefYesAttNum


\newcommand{\GetNumVal}[4]{%
  \pgfmathtruncatemacro{\row}{#4-1}%
  \pgfplotstablegetelem{\row}{#3}\of#2%
  \edef#1{\pgfplotsretval}%
}

\newcommand{\percentage}[2]{%
  \pgfmathparse{100*(#2/#1 - 1)}%
  \pgfmathprintnumber[fixed,precision=1,showpos]{\pgfmathresult}%
}


\newcommand{\csvConstNoDefeNoADSYesAtt}[2]{%
    \pgfplotstablegetelem{\numexpr#2-1}{#1}\of{\ConstNoDefeNoADSYesAtt}%
    \pgfplotsretval%
}
\newcommand{\csvNoDefeNoADSYesAtt}[2]{%
    \pgfplotstablegetelem{\numexpr#2-1}{#1}\of{\ExpNoDefeNoADSYesAttStr}%
    \pgfplotsretval%
}
\newcommand{\csvNoDefeNoADSYesAttNum}[2]{%
    \pgfplotstablegetelem{\numexpr#2-1}{#1}\of{\ExpNoDefeNoADSYesAttNum}%
    \pgfplotsretval%
}

\newcommand{\csvConstNoDefeYesAtt}[2]{%
    \pgfplotstablegetelem{\numexpr#2-1}{#1}\of{\ConstNoDefeYesAtt}%
    \pgfplotsretval%
}
\newcommand{\csvNoDefeYesAtt}[2]{%
    \pgfplotstablegetelem{\numexpr#2-1}{#1}\of{\ExpNoDefeYesAttStr}%
    \pgfplotsretval%
}
\newcommand{\csvNoDefeYesAttNum}[2]{%
    \pgfplotstablegetelem{\numexpr#2-1}{#1}\of{\ExpNoDefeYesAttNum}%
    \pgfplotsretval%
}

\newcommand{\csvConstYesDefYesAtt}[2]{%
    \pgfplotstablegetelem{\numexpr#2-1}{#1}\of{\ConstYesDefYesAtt}%
    \pgfplotsretval%
}
\newcommand{\csvYesDefYesAtt}[2]{%
    \pgfplotstablegetelem{\numexpr#2-1}{#1}\of{\ExpYesDefYesAttStr}%
    \pgfplotsretval%
}
\newcommand{\csvYesDefYesAttNum}[2]{%
    \pgfplotstablegetelem{\numexpr#2-1}{#1}\of{\ExpYesDefYesAttNum}%
    \pgfplotsretval%
}


\GetNumVal{\NoDefeNoADSYesAttFinalNormHandVelocity}{\ExpNoDefeNoADSYesAttNum}{FinalNorm_HandVelocity}{1}
\GetNumVal{\NoDefeYesAttFinalNormHandVelocity}{\ExpNoDefeYesAttNum}{FinalNorm_HandVelocity}{1}
\GetNumVal{\YesDefYesAttFinalNormHandVelocity}{\ExpYesDefYesAttNum}{FinalNorm_HandVelocity}{1}

\GetNumVal{\NoDefeNoADSYesAttFinalNormSysRealQdot}{\ExpNoDefeNoADSYesAttNum}{FinalNorm_sys_real_qdot}{1}
\GetNumVal{\NoDefeYesAttFinalNormSysRealQdot}{\ExpNoDefeYesAttNum}{FinalNorm_sys_real_qdot}{1}
\GetNumVal{\YesDefYesAttFinalNormSysRealQdot}{\ExpYesDefYesAttNum}{FinalNorm_sys_real_qdot}{1}

\GetNumVal{\NoDefeNoADSYesAttFinalNormUMax}{\ExpNoDefeNoADSYesAttNum}{FinalNorm_u_max}{1}
\GetNumVal{\NoDefeYesAttFinalNormUMax}{\ExpNoDefeYesAttNum}{FinalNorm_u_max}{1}
\GetNumVal{\YesDefYesAttFinalNormUMax}{\ExpYesDefYesAttNum}{FinalNorm_u_max}{1}

\GetNumVal{\NoDefeNoADSYesAttTs}{\ExpNoDefeNoADSYesAttNum}{Ts}{1}
\GetNumVal{\NoDefeYesAttTs}{\ExpNoDefeYesAttNum}{Ts}{1}
\GetNumVal{\YesDefYesAttTs}{\ExpYesDefYesAttNum}{Ts}{1}

\GetNumVal{\NoDefeNoADSYesAttAttChamferActToRef}{\ExpNoDefeNoADSYesAttNum}{att_chamferActToRef}{1}
\GetNumVal{\NoDefeYesAttAttChamferActToRef}{\ExpNoDefeYesAttNum}{att_chamferActToRef}{1}
\GetNumVal{\YesDefYesAttAttChamferActToRef}{\ExpYesDefYesAttNum}{att_chamferActToRef}{1}

\GetNumVal{\NoDefeNoADSYesAttAttChamferRefToAct}{\ExpNoDefeNoADSYesAttNum}{att_chamferRefToAct}{1}
\GetNumVal{\NoDefeYesAttAttChamferRefToAct}{\ExpNoDefeYesAttNum}{att_chamferRefToAct}{1}
\GetNumVal{\YesDefYesAttAttChamferRefToAct}{\ExpYesDefYesAttNum}{att_chamferRefToAct}{1}

\GetNumVal{\NoDefeNoADSYesAttAttChamferSym}{\ExpNoDefeNoADSYesAttNum}{att_chamferSym}{1}
\GetNumVal{\NoDefeYesAttAttChamferSym}{\ExpNoDefeYesAttNum}{att_chamferSym}{1}
\GetNumVal{\YesDefYesAttAttChamferSym}{\ExpYesDefYesAttNum}{att_chamferSym}{1}

\GetNumVal{\NoDefeNoADSYesAttAttFrechetDist}{\ExpNoDefeNoADSYesAttNum}{att_frechetDist}{1}
\GetNumVal{\NoDefeYesAttAttFrechetDist}{\ExpNoDefeYesAttNum}{att_frechetDist}{1}
\GetNumVal{\YesDefYesAttAttFrechetDist}{\ExpYesDefYesAttNum}{att_frechetDist}{1}

\GetNumVal{\NoDefeNoADSYesAttAttMaxError}{\ExpNoDefeNoADSYesAttNum}{att_maxError}{1}
\GetNumVal{\NoDefeYesAttAttMaxError}{\ExpNoDefeYesAttNum}{att_maxError}{1}
\GetNumVal{\YesDefYesAttAttMaxError}{\ExpYesDefYesAttNum}{att_maxError}{1}

\GetNumVal{\NoDefeNoADSYesAttAttRmsError}{\ExpNoDefeNoADSYesAttNum}{att_rmsError}{1}
\GetNumVal{\NoDefeYesAttAttRmsError}{\ExpNoDefeYesAttNum}{att_rmsError}{1}
\GetNumVal{\YesDefYesAttAttRmsError}{\ExpYesDefYesAttNum}{att_rmsError}{1}

\GetNumVal{\NoDefeNoADSYesAttAttackerDeviationHandPLTwoAvg}{\ExpNoDefeNoADSYesAttNum}{attackerDeviation_HandP_l2_avg}{1}
\GetNumVal{\NoDefeYesAttAttackerDeviationHandPLTwoAvg}{\ExpNoDefeYesAttNum}{attackerDeviation_HandP_l2_avg}{1}
\GetNumVal{\YesDefYesAttAttackerDeviationHandPLTwoAvg}{\ExpYesDefYesAttNum}{attackerDeviation_HandP_l2_avg}{1}

\GetNumVal{\NoDefeNoADSYesAttAttackerDeviationHandPLTwoMax}{\ExpNoDefeNoADSYesAttNum}{attackerDeviation_HandP_l2_max}{1}
\GetNumVal{\NoDefeYesAttAttackerDeviationHandPLTwoMax}{\ExpNoDefeYesAttNum}{attackerDeviation_HandP_l2_max}{1}
\GetNumVal{\YesDefYesAttAttackerDeviationHandPLTwoMax}{\ExpYesDefYesAttNum}{attackerDeviation_HandP_l2_max}{1}

\GetNumVal{\NoDefeNoADSYesAttAttackerDeviationHandPLTwoTot}{\ExpNoDefeNoADSYesAttNum}{attackerDeviation_HandP_l2_tot}{1}
\GetNumVal{\NoDefeYesAttAttackerDeviationHandPLTwoTot}{\ExpNoDefeYesAttNum}{attackerDeviation_HandP_l2_tot}{1}
\GetNumVal{\YesDefYesAttAttackerDeviationHandPLTwoTot}{\ExpYesDefYesAttNum}{attackerDeviation_HandP_l2_tot}{1}

\GetNumVal{\NoDefeNoADSYesAttDeviationHandPLTwoAvg}{\ExpNoDefeNoADSYesAttNum}{deviation_HandP_l2_avg}{1}
\GetNumVal{\NoDefeYesAttDeviationHandPLTwoAvg}{\ExpNoDefeYesAttNum}{deviation_HandP_l2_avg}{1}
\GetNumVal{\YesDefYesAttDeviationHandPLTwoAvg}{\ExpYesDefYesAttNum}{deviation_HandP_l2_avg}{1}

\GetNumVal{\NoDefeNoADSYesAttDeviationHandPLTwoMax}{\ExpNoDefeNoADSYesAttNum}{deviation_HandP_l2_max}{1}
\GetNumVal{\NoDefeYesAttDeviationHandPLTwoMax}{\ExpNoDefeYesAttNum}{deviation_HandP_l2_max}{1}
\GetNumVal{\YesDefYesAttDeviationHandPLTwoMax}{\ExpYesDefYesAttNum}{deviation_HandP_l2_max}{1}

\GetNumVal{\NoDefeNoADSYesAttDeviationHandPLTwoTot}{\ExpNoDefeNoADSYesAttNum}{deviation_HandP_l2_tot}{1}
\GetNumVal{\NoDefeYesAttDeviationHandPLTwoTot}{\ExpNoDefeYesAttNum}{deviation_HandP_l2_tot}{1}
\GetNumVal{\YesDefYesAttDeviationHandPLTwoTot}{\ExpYesDefYesAttNum}{deviation_HandP_l2_tot}{1}

\GetNumVal{\NoDefeNoADSYesAttDeviationXEstxRealLTwoAvg}{\ExpNoDefeNoADSYesAttNum}{deviation_xEstxReal_l2_avg}{1}
\GetNumVal{\NoDefeYesAttDeviationXEstxRealLTwoAvg}{\ExpNoDefeYesAttNum}{deviation_xEstxReal_l2_avg}{1}
\GetNumVal{\YesDefYesAttDeviationXEstxRealLTwoAvg}{\ExpYesDefYesAttNum}{deviation_xEstxReal_l2_avg}{1}

\GetNumVal{\NoDefeNoADSYesAttDeviationXEstxRealLTwoMax}{\ExpNoDefeNoADSYesAttNum}{deviation_xEstxReal_l2_max}{1}
\GetNumVal{\NoDefeYesAttDeviationXEstxRealLTwoMax}{\ExpNoDefeYesAttNum}{deviation_xEstxReal_l2_max}{1}
\GetNumVal{\YesDefYesAttDeviationXEstxRealLTwoMax}{\ExpYesDefYesAttNum}{deviation_xEstxReal_l2_max}{1}

\GetNumVal{\NoDefeNoADSYesAttDeviationXEstxRealLTwoTot}{\ExpNoDefeNoADSYesAttNum}{deviation_xEstxReal_l2_tot}{1}
\GetNumVal{\NoDefeYesAttDeviationXEstxRealLTwoTot}{\ExpNoDefeYesAttNum}{deviation_xEstxReal_l2_tot}{1}
\GetNumVal{\YesDefYesAttDeviationXEstxRealLTwoTot}{\ExpYesDefYesAttNum}{deviation_xEstxReal_l2_tot}{1}

\GetNumVal{\NoDefeNoADSYesAttDeviationXSFPxRealLTwoAvg}{\ExpNoDefeNoADSYesAttNum}{deviation_xSFPxReal_l2_avg}{1}
\GetNumVal{\NoDefeYesAttDeviationXSFPxRealLTwoAvg}{\ExpNoDefeYesAttNum}{deviation_xSFPxReal_l2_avg}{1}
\GetNumVal{\YesDefYesAttDeviationXSFPxRealLTwoAvg}{\ExpYesDefYesAttNum}{deviation_xSFPxReal_l2_avg}{1}

\GetNumVal{\NoDefeNoADSYesAttDeviationXSFPxRealLTwoMax}{\ExpNoDefeNoADSYesAttNum}{deviation_xSFPxReal_l2_max}{1}
\GetNumVal{\NoDefeYesAttDeviationXSFPxRealLTwoMax}{\ExpNoDefeYesAttNum}{deviation_xSFPxReal_l2_max}{1}
\GetNumVal{\YesDefYesAttDeviationXSFPxRealLTwoMax}{\ExpYesDefYesAttNum}{deviation_xSFPxReal_l2_max}{1}

\GetNumVal{\NoDefeNoADSYesAttDeviationXSFPxRealLTwoTot}{\ExpNoDefeNoADSYesAttNum}{deviation_xSFPxReal_l2_tot}{1}
\GetNumVal{\NoDefeYesAttDeviationXSFPxRealLTwoTot}{\ExpNoDefeYesAttNum}{deviation_xSFPxReal_l2_tot}{1}
\GetNumVal{\YesDefYesAttDeviationXSFPxRealLTwoTot}{\ExpYesDefYesAttNum}{deviation_xSFPxReal_l2_tot}{1}

\GetNumVal{\YesDefYesAttGsAvg}{\ExpYesDefYesAttNum}{gsAvg}{1}

\GetNumVal{\YesDefYesAttGsMax}{\ExpYesDefYesAttNum}{gsMax}{1}

\GetNumVal{\YesDefYesAttGsMin}{\ExpYesDefYesAttNum}{gsMin}{1}

\GetNumVal{\NoDefeNoADSYesAttHandPathLength}{\ExpNoDefeNoADSYesAttNum}{handPathLength}{1}
\GetNumVal{\NoDefeYesAttHandPathLength}{\ExpNoDefeYesAttNum}{handPathLength}{1}
\GetNumVal{\YesDefYesAttHandPathLength}{\ExpYesDefYesAttNum}{handPathLength}{1}

\GetNumVal{\NoDefeNoADSYesAttHandChamferActToRef}{\ExpNoDefeNoADSYesAttNum}{hand_chamferActToRef}{1}
\GetNumVal{\NoDefeYesAttHandChamferActToRef}{\ExpNoDefeYesAttNum}{hand_chamferActToRef}{1}
\GetNumVal{\YesDefYesAttHandChamferActToRef}{\ExpYesDefYesAttNum}{hand_chamferActToRef}{1}

\GetNumVal{\NoDefeNoADSYesAttHandChamferRefToAct}{\ExpNoDefeNoADSYesAttNum}{hand_chamferRefToAct}{1}
\GetNumVal{\NoDefeYesAttHandChamferRefToAct}{\ExpNoDefeYesAttNum}{hand_chamferRefToAct}{1}
\GetNumVal{\YesDefYesAttHandChamferRefToAct}{\ExpYesDefYesAttNum}{hand_chamferRefToAct}{1}

\GetNumVal{\NoDefeNoADSYesAttHandChamferSym}{\ExpNoDefeNoADSYesAttNum}{hand_chamferSym}{1}
\GetNumVal{\NoDefeYesAttHandChamferSym}{\ExpNoDefeYesAttNum}{hand_chamferSym}{1}
\GetNumVal{\YesDefYesAttHandChamferSym}{\ExpYesDefYesAttNum}{hand_chamferSym}{1}

\GetNumVal{\NoDefeNoADSYesAttHandFrechetDist}{\ExpNoDefeNoADSYesAttNum}{hand_frechetDist}{1}
\GetNumVal{\NoDefeYesAttHandFrechetDist}{\ExpNoDefeYesAttNum}{hand_frechetDist}{1}
\GetNumVal{\YesDefYesAttHandFrechetDist}{\ExpYesDefYesAttNum}{hand_frechetDist}{1}

\GetNumVal{\NoDefeNoADSYesAttHandLTwoAvg}{\ExpNoDefeNoADSYesAttNum}{hand_l2_avg}{1}
\GetNumVal{\NoDefeYesAttHandLTwoAvg}{\ExpNoDefeYesAttNum}{hand_l2_avg}{1}
\GetNumVal{\YesDefYesAttHandLTwoAvg}{\ExpYesDefYesAttNum}{hand_l2_avg}{1}

\GetNumVal{\NoDefeNoADSYesAttHandLTwoMax}{\ExpNoDefeNoADSYesAttNum}{hand_l2_max}{1}
\GetNumVal{\NoDefeYesAttHandLTwoMax}{\ExpNoDefeYesAttNum}{hand_l2_max}{1}
\GetNumVal{\YesDefYesAttHandLTwoMax}{\ExpYesDefYesAttNum}{hand_l2_max}{1}

\GetNumVal{\NoDefeNoADSYesAttHandLTwoTot}{\ExpNoDefeNoADSYesAttNum}{hand_l2_tot}{1}
\GetNumVal{\NoDefeYesAttHandLTwoTot}{\ExpNoDefeYesAttNum}{hand_l2_tot}{1}
\GetNumVal{\YesDefYesAttHandLTwoTot}{\ExpYesDefYesAttNum}{hand_l2_tot}{1}

\GetNumVal{\NoDefeNoADSYesAttHandMaxError}{\ExpNoDefeNoADSYesAttNum}{hand_maxError}{1}
\GetNumVal{\NoDefeYesAttHandMaxError}{\ExpNoDefeYesAttNum}{hand_maxError}{1}
\GetNumVal{\YesDefYesAttHandMaxError}{\ExpYesDefYesAttNum}{hand_maxError}{1}

\GetNumVal{\NoDefeNoADSYesAttHandRmsError}{\ExpNoDefeNoADSYesAttNum}{hand_rmsError}{1}
\GetNumVal{\NoDefeYesAttHandRmsError}{\ExpNoDefeYesAttNum}{hand_rmsError}{1}
\GetNumVal{\YesDefYesAttHandRmsError}{\ExpYesDefYesAttNum}{hand_rmsError}{1}

\GetNumVal{\NoDefeNoADSYesAttPowerPDLOneAvg}{\ExpNoDefeNoADSYesAttNum}{power_PD_l1_avg}{1}
\GetNumVal{\NoDefeYesAttPowerPDLOneAvg}{\ExpNoDefeYesAttNum}{power_PD_l1_avg}{1}
\GetNumVal{\YesDefYesAttPowerPDLOneAvg}{\ExpYesDefYesAttNum}{power_PD_l1_avg}{1}

\GetNumVal{\NoDefeNoADSYesAttPowerPDLOneMax}{\ExpNoDefeNoADSYesAttNum}{power_PD_l1_max}{1}
\GetNumVal{\NoDefeYesAttPowerPDLOneMax}{\ExpNoDefeYesAttNum}{power_PD_l1_max}{1}
\GetNumVal{\YesDefYesAttPowerPDLOneMax}{\ExpYesDefYesAttNum}{power_PD_l1_max}{1}

\GetNumVal{\NoDefeNoADSYesAttPowerPDLOneTot}{\ExpNoDefeNoADSYesAttNum}{power_PD_l1_tot}{1}
\GetNumVal{\NoDefeYesAttPowerPDLOneTot}{\ExpNoDefeYesAttNum}{power_PD_l1_tot}{1}
\GetNumVal{\YesDefYesAttPowerPDLOneTot}{\ExpYesDefYesAttNum}{power_PD_l1_tot}{1}

\GetNumVal{\NoDefeNoADSYesAttPowerFinalLOneAvg}{\ExpNoDefeNoADSYesAttNum}{power_final_l1_avg}{1}
\GetNumVal{\NoDefeYesAttPowerFinalLOneAvg}{\ExpNoDefeYesAttNum}{power_final_l1_avg}{1}
\GetNumVal{\YesDefYesAttPowerFinalLOneAvg}{\ExpYesDefYesAttNum}{power_final_l1_avg}{1}

\GetNumVal{\NoDefeNoADSYesAttPowerFinalLOneMax}{\ExpNoDefeNoADSYesAttNum}{power_final_l1_max}{1}
\GetNumVal{\NoDefeYesAttPowerFinalLOneMax}{\ExpNoDefeYesAttNum}{power_final_l1_max}{1}
\GetNumVal{\YesDefYesAttPowerFinalLOneMax}{\ExpYesDefYesAttNum}{power_final_l1_max}{1}

\GetNumVal{\NoDefeNoADSYesAttPowerFinalLOneTot}{\ExpNoDefeNoADSYesAttNum}{power_final_l1_tot}{1}
\GetNumVal{\NoDefeYesAttPowerFinalLOneTot}{\ExpNoDefeYesAttNum}{power_final_l1_tot}{1}
\GetNumVal{\YesDefYesAttPowerFinalLOneTot}{\ExpYesDefYesAttNum}{power_final_l1_tot}{1}

\GetNumVal{\NoDefeNoADSYesAttSysRealQdotLTwoAvg}{\ExpNoDefeNoADSYesAttNum}{sys_real_qdot_l2_avg}{1}
\GetNumVal{\NoDefeYesAttSysRealQdotLTwoAvg}{\ExpNoDefeYesAttNum}{sys_real_qdot_l2_avg}{1}
\GetNumVal{\YesDefYesAttSysRealQdotLTwoAvg}{\ExpYesDefYesAttNum}{sys_real_qdot_l2_avg}{1}

\GetNumVal{\NoDefeNoADSYesAttSysRealQdotLTwoMax}{\ExpNoDefeNoADSYesAttNum}{sys_real_qdot_l2_max}{1}
\GetNumVal{\NoDefeYesAttSysRealQdotLTwoMax}{\ExpNoDefeYesAttNum}{sys_real_qdot_l2_max}{1}
\GetNumVal{\YesDefYesAttSysRealQdotLTwoMax}{\ExpYesDefYesAttNum}{sys_real_qdot_l2_max}{1}

\GetNumVal{\NoDefeNoADSYesAttSysRealQdotLTwoTot}{\ExpNoDefeNoADSYesAttNum}{sys_real_qdot_l2_tot}{1}
\GetNumVal{\NoDefeYesAttSysRealQdotLTwoTot}{\ExpNoDefeYesAttNum}{sys_real_qdot_l2_tot}{1}
\GetNumVal{\YesDefYesAttSysRealQdotLTwoTot}{\ExpYesDefYesAttNum}{sys_real_qdot_l2_tot}{1}

\GetNumVal{\NoDefeNoADSYesAttUFinalLTwoAvg}{\ExpNoDefeNoADSYesAttNum}{u_final_l2_avg}{1}
\GetNumVal{\NoDefeYesAttUFinalLTwoAvg}{\ExpNoDefeYesAttNum}{u_final_l2_avg}{1}
\GetNumVal{\YesDefYesAttUFinalLTwoAvg}{\ExpYesDefYesAttNum}{u_final_l2_avg}{1}

\GetNumVal{\NoDefeNoADSYesAttUFinalLTwoTot}{\ExpNoDefeNoADSYesAttNum}{u_final_l2_tot}{1}
\GetNumVal{\NoDefeYesAttUFinalLTwoTot}{\ExpNoDefeYesAttNum}{u_final_l2_tot}{1}
\GetNumVal{\YesDefYesAttUFinalLTwoTot}{\ExpYesDefYesAttNum}{u_final_l2_tot}{1}

\maketitle

\begin{abstract}
Cyber-physical robotic systems are vulnerable to \emph{false data injection attacks} (FDIAs), in which an adversary corrupts sensor signals while evading residual-based \emph{passive} anomaly detectors such as the $\chi^2$ test. Such \emph{stealthy} attacks can induce substantial end-effector deviations without triggering alarms. This paper studies the resilience of redundant manipulators to stealthy FDIAs and advances the architecture from passive monitoring to active defence. We formulate a closed-loop model comprising a feedback-linearized manipulator, a steady-state Kalman filter, and a $\chi^2$-based anomaly detector. Building on this passive monitoring layer, we propose an active control-level defence that attenuates the control input through a monotone function of an anomaly score generated by a novel actuation-projected, measurement-free state predictor. The proposed design provides probabilistic guarantees on nominal actuation loss and preserves closed-loop stability. From the attacker perspective, we derive a convex QCQP for computing one-step optimal stealthy attacks. Simulations on a 6-DOF planar manipulator show that the proposed defence significantly reduces attack-induced end-effector deviation while preserving nominal task performance in the absence of attacks.
\end{abstract}

\section{Introduction}\label{sec:intro}
Robotic manipulators are increasingly deployed in open and networked environments, ranging from industrial assembly to collaborative human-robot interaction.  
Their tight integration of computation, communication, and actuation, however, makes them vulnerable to cyberattacks that directly compromise safety and reliability.  
Among cyber-physical threats, attacks on \emph{data integrity} are particularly critical: by corrupting sensor information, an adversary can mislead the controller and alter the robot's behaviour without any physical contact~\cite{humayed_cyber-physical_2017, sandberg_secure_2022}.  
Such threats are especially concerning when they are \emph{stealthy}, i.e., engineered to remain below the threshold of an anomaly detection system (ADS), thereby avoiding alarms while still driving the end-effector away from its intended task~\cite{guo_optimal_2017, intriago_residual-based_2024}.

A well-studied class of integrity attacks is \emph{False Data Injection Attacks} (FDIAs), in which adversaries manipulate sensor signals to achieve malicious objectives.  
In networked control and power systems, stealthy FDIAs have been extensively analyzed: from characterizations of undetectability~\cite{mo_secure_2010, fawzi_secure_2014, ueda_affine_2024} to optimal attack synthesis.  
In robotics, however, prior work has primarily focused on passive anomaly detection or on ``perfectly undetectable'' attacks that completely bypass detection~\cite{ueda_affine_2024}.  
These perspectives miss a key robotics-specific vulnerability: the widespread use of \emph{feedback linearization} reduces manipulator dynamics to double integrators, which in turn induces an \emph{integrator vulnerability}.  
As a consequence, persistent sensor corruption can silently accumulate in the closed-loop system, driving large task-space errors even under residual-based $\chi^2$ detection~\cite{tosun2025kullback}.

\textbf{This paper addresses this gap by introducing an \emph{active defence} strategy that transforms anomaly detection from passive monitoring into a resilience mechanism.}
The proposed method leverages an \emph{actuation-projected state predictor}, a model-driven state estimate that ignores sensing, to compute an anomaly score immune to direct sensor corruption.  
Based on this score, we introduce \emph{anomaly-aware command scaling}, which attenuates commanded accelerations as anomalies grow, thereby reducing the adversary's ability to steer the manipulator while preserving nominal performance.  

\paragraph*{Contributions.}
The contributions of this work are the following:
\begin{enumerate}
    \item We formalize stealthy false data injection attacks (FDIAs) against the sensors of feedback-linearized manipulators, exposing their integrator vulnerability and showing that the attacker's one-step optimal strategy reduces to a convex QCQP.
    
    \item We propose \emph{anomaly-aware command scaling}, which attenuates control inputs based on a measurement-free actuation-projected state predictor.
    
    \item We provide probabilistic guarantees on bounded attenuation in nominal operation and prove closed-loop stability under the proposed defence.
    
    \item Simulations on a 6-DOF redundant manipulator demonstrate that the defence significantly limits attacker-realizable end-effector deviations while preserving task performance in the absence of attacks.
\end{enumerate}
\subsection{Related Work}\label{sec:related_work}

\textbf{Anomaly detection in CPS.}  
Residual-based detectors, particularly $\chi^2$ tests on Kalman innovations, are a classical tool for monitoring CPS integrity \cite{ding_survey_2018}.  
These methods provide statistical guarantees on false-alarm rates and are widely used in industrial practice.  
Extensions include adaptive thresholds \cite{r_tuning_2018} and sequential schemes such as CUSUM tests \cite{murguia_cusum_2016}.  
However, such ADS are inherently \emph{passive}: they detect anomalies but do not alter the control policy, leaving the system structure unchanged.  

\textbf{False Data Injection Attacks.}  
Theoretical studies of FDIAs in networked systems have characterized undetectability conditions \cite{mo_secure_2010, fawzi_secure_2014}, affine attack structures \cite{ueda_affine_2024}, and optimal attack strategies \cite{guo_optimal_2017}.  
These works typically assume either ``perfectly undetectable'' attacks (no residual information leaks) or focus on power networks and generic LTI plants.  
Robotics-specific studies remain limited: most works concentrate on detector design rather than on modifying the controller to actively limit attack effectiveness \cite{intriago_residual-based_2024}.

\textbf{Integrator Vulnerability.}
Control systems with an LTI plant having integral action are subject to integrator vulnerability~\cite{tosun2025kullback}, where the residual generated by any linear observer follows the same distribution during normal operation and under attack in the steady-state regime, making sensor bias injection attacks detectable only during transients.

\textbf{Our contribution.}
We bridge these lines of research by combining a residual-based ADS with a novel active defence operating at the control level.
First, we show that feedback linearization induces an \emph{integrator vulnerability}: a PD-based FDIA can precisely steer the end-effector despite a residual-based $\chi^2$ detector.
We then show that our gain-scaling defence, driven by the anomaly score, reduces the attack’s effectiveness while providing probabilistic guarantees in attack-free operation, thereby improving the cyber-physical security of robotic manipulators.

\section{System Model}\label{sec:system_model}
We consider a robotic manipulator operating in closed loop with a state estimator, a task-space controller, and an anomaly detection system (ADS).  
The architecture, illustrated in Fig.~\ref{fig:control_attack}, captures the interaction between defender and adversary: (i) the plant dynamics, (ii) a Kalman filter for state estimation, (iii) task-space control, (iv) a residual-based ADS, and (v) an additive adversarial attack on sensor measurements.

\subsection{Closed-Loop Architecture}
At each discrete time step $k \in \mathbb{Z}_{\ge 0}$, the plant output $\bm{y}_k \in \mathbb{R}^p$ is corrupted by an injected signal $\bm{a}_k$ to yield the attacked measurement
\begin{equation}
    \tilde{\bm{y}}_k = \bm{y}_k + \bm{a}_k,
    \label{eq:measurement_model}
\end{equation}

\begin{figure}[t]
    \centering
    \vspace{0.6em}
    \adjustbox{max width=\columnwidth, max height=\textheight}{\includegraphics{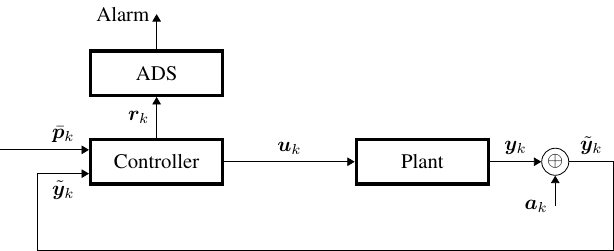}}
    \caption{Closed-loop system under FDIA. Sensor signals are corrupted by injection $\bm{a}_k$, yielding attacked output $\tilde{\bm{y}}_k$. The Kalman filter generates the innovation $\bm{r}_k$, which is monitored by the ADS.}
    \label{fig:control_attack}
\end{figure}

\subsection{Plant and State Estimator}
The plant is modeled as a discrete-time LTI system:
\begin{subequations}\label{eq:lti_plant}
\begin{align}
    \bm{x}_{k+1} &= \bm{A}\bm{x}_k + \bm{B}\bm{u}_k + \bm{w}_k, \\
    \bm{y}_k &= \bm{C}\bm{x}_k + \bm{v}_k,
\end{align}
\end{subequations}
with state $\bm{x}_k \in \mathbb{R}^{n}$, process noise 
$\ProcessNoise_k \sim \mathcal{N}(\bm{0},\bm{Q})$, and measurement noise 
$\MeasurementNoise_k \sim \mathcal{N}(\bm{0},\bm{R})$.
We assume $(\bm{A},\bm{C})$ is detectable,
ensuring the existence of a unique stabilizing solution $\bm{P}$ to the DARE
\begin{equation*}
    \bm{P} = \bm{A}\bm{P}\bm{A}^\top + \bm{Q}
    - \bm{A}\bm{P}\bm{C}^\top\!\left(\bm{C}\bm{P}\bm{C}^\top + \bm{R}\right)^{-1}\!
    \bm{C}\bm{P}\bm{A}^\top .
\end{equation*}

State estimates are obtained from a steady-state Kalman filter in single-step innovation form:
\begin{subequations}
\begin{align}
    \hat{\bm{x}}_{k+1} &= \bm{A}\hat{\bm{x}}_{k} + \bm{B}\bm{u}_k + \KalmanGain\,\bm{r}_k, \label{eq:kalmanX}  \\
    \bm{r}_k &= \tilde{\bm{y}}_{k} - \bm{C}\hat{\bm{x}}_{k}, \label{eq:kalmanR}
\end{align}
\end{subequations}
where $\bm{r}_k \in \mathbb{R}^{p}$ is the innovation. 
Its steady-state covariance is 
{${\bm{\Sigma} = \bm{C}\bm{P}\bm{C}^\top + \bm{R}}$}, and the corresponding 
steady-state gain is
\begin{equation}
    \KalmanGain = \bm{A}\,\bm{P}\,\bm{C}^\top\,\bm{\Sigma}^{-1}.
\end{equation}
The residual $\bm{r}_k$ and its covariance $\bm{\Sigma}$ will later be used 
for residual-based anomaly detection.

\subsection{Manipulator Model and Task-Space Control}
We consider a $p$-DOF robotic manipulator with sensor readings providing joint values $\bm{q}$, and having state
{${\bm{x} = [\bm{q}^\top \; \dot{\bm{q}}^\top]^\top \in \mathbb{R}^{n}}$}, where $n=2 \cdot p$ and $\bm{q}\in\mathbb{R}^{p}$. 

The continuous-time joint-space dynamics are
\begin{equation}
    \bm{M}(\bm{q})\,\ddot{\bm{q}} + \bm{\nu}(\bm{q},\dot{\bm{q}}) = \bm{\tau},
    \label{eq:manipulator_dynamics}
\end{equation}
where $\bm{M}(\bm{q})\succ 0$ is the inertia matrix and $\bm{\nu}$ collects Coriolis, centrifugal, and gravity terms. With inverse-dynamics compensation, i.e., $\bm{\tau}=\bm{M}(\bm{q})\,\bm{u}^{\text{nom}}+\bm{\nu}(\bm{q},\dot{\bm{q}})$, the joint dynamics reduce to decoupled double integrators,
\begin{equation}\label{eq:idealDynamics}
    \ddot{\bm{q}} = \bm{u}^{\text{nom}},
\end{equation}
whose discrete-time form is represented by \cref{eq:lti_plant} with {${\bm{x}_k=[\bm{q}_k^\top \; \dot{\bm{q}}_k^\top]^\top}$} and input $\bm{u}_k=\bm{u}^{\text{nom}}_k$.

The controller operates in task space using twist control \cite{murray_mathematical_1994,siciliano_robotics_2010}, tracking position references $\{\TargetPosition_k,\dot{\TargetPosition}_k,\ddot{\TargetPosition}_k\}$ and orientation references $\{\TargetRotationMatrix_k,\TargetAngularVelocity_k,\dot{\TargetAngularVelocity}_k\}$, where $\TargetRotationMatrix_k\in SO(3)$ and $\TargetAngularVelocity_k\in\mathbb{R}^3$. 

Position and orientation errors are defined as
\begin{align}
    \ErrorPosition{k} &= \TargetPosition_k - \PositionEst_k, \label{eq:position_error}\\
    \ErrorOrientation{k} &= \sin\!\left(\tfrac{\AngleAxisAngle_k}{2}\right)\,\AngleAxisAxis_k, \label{eq:orientation_error}
\end{align}
where $(\AngleAxisAngle_k,\AngleAxisAxis_k)$ is the angle–axis pair of the rotation error $\TargetRotationMatrix_k\,\EstimatedRotationMatrix_k^\top$ with $\AngleAxisAngle_k\in[0,\pi]$.
The estimates $\PositionEst_k$ and $\EstimatedRotationMatrix_k$ are obtained via forward kinematics from $\hat{\bm{q}}_k$ extracted from the state estimate $\hat{\bm{x}}_k$.


A PD+feedforward law generates desired task accelerations:
\begin{equation}
    \label{eq:pid_actuation_task}
    \bm{u}^{\text{c}}_{k} = 
    \begin{bmatrix}
        \ddot{\TargetPosition}_k 
            + \GainProportionalPosition \,\ErrorPosition{k} 
            + \GainDerivativePosition \,\ErrorPositionDerivative{k} \\[6pt]
        \dot{\TargetAngularVelocity}_k 
            + \GainProportionalOrientation \,\ErrorOrientation{k} 
            + \GainDerivativeOrientation \,\ErrorAngular{k}
    \end{bmatrix},
\end{equation}
where $\ErrorPositionDerivative{k} \!\coloneqq\! \dot{\TargetPosition}_k - \dot{\PositionEst}_k$ and 
$\ErrorAngular{k} \!\coloneqq\! \TargetAngularVelocity_k - \EstimatedAngularVelocity_k$ are velocity errors. 
The gains are synthesized via discrete-time LQR on the discretized double-integrator model \eqref{eq:idealDynamics}.

Joint accelerations are then computed via the pseudoinverse of the geometric Jacobian $\bm{J}(\bm{q})\in\mathbb{R}^{6\times p}$~\cite{siciliano_robotics_2010}:
\newcommand{\pidActuationJoint}
 {\bm{J}^\dagger(\bm{q})\big( \bm{u}^{\text{c}}_{k} - \dot{\bm{J}}(\bm{q},\dot{\bm{q}})\dot{\bm{q}}\big)
 }
\begin{equation}
    \label{eq:pid_actuation_joint}
    \bm{u}^{\text{nom}}_k = \pidActuationJoint.
\end{equation}

\subsection{Passive Anomaly Detection System}\label{subsec:ADS}
The ADS monitors the innovation sequence $\bm{r}_k$. 
With steady-state covariance $\bm{\Sigma}\succ 0$ (hence invertible), the Mahalanobis distance
\begin{equation}
    z_k = \bm{r}_k^\top \bm{\Sigma}^{-1} \bm{r}_k
    \label{eq:mahalanobis_distance}
\end{equation}
is $\chi^2(p)$-distributed under $\mathcal{H}_0$ (no attack), where {${p=\dim(\bm{y}_k)}$}.
To reduce sensitivity to single-sample fluctuations, we employ a windowed statistic of length $W\in\mathbb{N}$:
\begin{equation}
    w_k = \sum_{i=k-W+1}^k z_i.
    \label{eq:mahalanobis_window}
\end{equation}
During execution, an alarm is triggered if $w_k > \tau$.

\begin{lemma}[Innovation whiteness]\label{lem:innovation}
Under $\mathcal{H}_0$ (no attack), the innovation process satisfies
\[
\bm{r}_k \sim \mathcal{N}(\bm{0},\bm{\Sigma}), \qquad \bm{r}_k \ \text{independent across } k.
\]
\end{lemma}

\begin{proof}
This is a standard property of the Kalman filter with correct noise statistics: the innovation sequence is white, zero-mean Gaussian with covariance $\bm{\Sigma}$, see \cite{ding_survey_2018}.
\end{proof}

\begin{lemma}[Chi-squared distribution]\label{thm:chi2}
Under $\mathcal{H}_0$, the statistics satisfy
\[
z_k \sim \chi^2(p), \qquad w_k \sim \chi^2(pW),
\]
where $p$ is the output dimension and $W$ the window length.
\end{lemma}
\begin{proof}
From Lemma~\ref{lem:innovation}, $\bm{\Sigma}^{-1/2} \bm{r}_k \sim \mathcal{N}(\bm{0},\bm{I}_p)$.  
Thus $z_k = \| \bm{\Sigma}^{-1/2} {\bm{r}}_k\|_2^2 \sim \chi^2(p)$ \cite{gallego_mahalanobis_2013}.  
Independence across time implies $w_k$ is the sum of $W$ i.i.d. $\chi^2(p)$ variables, hence $\chi^2(pW)$, see \cite{r_tuning_2018}.
\end{proof}

\begin{corollary}[Threshold calibration]\label{cor:threshold}
Given a desired per-step false-alarm probability $\alpha\in(0,1)$,
set
\begin{equation}
\tau \;=\; F^{-1}_{\chi^2(pW)}(1-\alpha)
\;=\; 2\, \mathrm{P}^{-1}\!\left(1-\alpha,\tfrac{pW}{2}\right),
\label{eq:desired_arl}
\end{equation}
which ensures $\Pr(w_k > \tau \mid \mathcal{H}_0)=\alpha$, where $\mathrm{P}^{-1}$ denotes the inverse of the regularized lower incomplete gamma function \cite{r_tuning_2018}.
\end{corollary}


\subsection{Attacker Model}\label{subsec:assumptions}
We now state the assumptions under which both the defence and attack strategies are developed.

\begin{assumption}[Adversary knowledge]\label{ass:omniscience}
The adversary is \emph{omniscient}: it has full knowledge of the plant, controller, state estimator, and defence system(s).
\end{assumption}

\begin{assumption}[Adversary capabilities]\label{ass:capab}
The only attack surface is additive injection into sensor measurements. 
The adversary cannot tamper with actuators, control logic, ADS parameters, timing, or packet ordering.
\end{assumption}

\begin{assumption}[Adversary goal]\label{ass:goal}
The adversary’s goal is to manipulate the position of the end-effector in task space, as formalized later by an attack objective on the task-space trajectory, subject to the detection mechanism.
\end{assumption}

\begin{assumption}[Converged estimator]\label{ass:converged}
Without loss of generality, the attack starts at $k=0$ and lasts at most $\AttackDuration$ samples. 
At $k=0$, the Kalman filter operates at steady state (covariances and gains constant), i.e., the estimator has converged.
\end{assumption}

\section{Proposed Defence Method}\label{sec:defence}

We now present a defence strategy based on gain scaling to reduce the impact of sensor FDIAs, with probabilistic guarantees on the performance loss under attack-free operation.
\subsection{Attack Estimation}
We define the \emph{actuation-projected state} $\ActuationProjectedState_{k}$ as the measurement-free state prediction driven solely by commanded inputs. 
To mitigate drift from model mismatch, this predictor is periodically re-synchronized with the Kalman estimate as ground truth. 
Without loss of generality, assume the most recent re-synchronization occurs at $k=0$, so
\begin{equation}
    \ActuationProjectedState_{0} = \hat{\bm{x}}_{0}.
    \label{eq:resync}
\end{equation}
For subsequent steps, the actuation-projected state follows the noise-free, open-loop predictor
\begin{equation}
    \ActuationProjectedState_{k+1} 
    = \bm{A}\,\ActuationProjectedState_{k} + \bm{B}\,\bm{u}_{k}, \qquad k \ge 0.
    \label{eq:actuation_projected_state}
\end{equation}
We define the \emph{actuation-projected residual} (distinct from the innovation $\bm{r}_k$) as
\begin{equation}
    \tilde{\bm{r}}_{k} \;\coloneqq\; \hat{\bm{x}}_{k} - \ActuationProjectedState_{k} \in \mathbb{R}^{n},   \qquad n = 2 p.
    \label{eq:residualProj}
\end{equation}

\begin{theorem}[Actuation-projected residual covariance]
	\label{thm:residual_cov}
The covariance of the residual in \cref{eq:residualProj} after 
$k$ steps from initialization, $\bm{\Sigma}_{\tilde{r}, k} \in \mathbb{R}^{n \times n}$, is the $(2,2)$ block of the matrix $\bm{P}_{z,k} \in \mathbb{R}^{2n \times 2n}$ obtained 
by iterating
	\begin{equation}\label{eq:covRecurs}
		\bm{P}_{z,j+1} = \bm{F} \bm{P}_{z,j} \bm{F}^\top + \bm{\Pi}, \quad j=0, \dots, k-1,
	\end{equation}
from the initial condition $\bm{P}_{z,0} = \operatorname{diag}(\bm{P}, \bm{0})$, where
\begin{align}
    \label{eq:covMatF}\bm{F} = \begin{bmatrix} \bm{A} - \KalmanGain \bm{C} & \bm{0} \\ \KalmanGain \bm{C} & \bm{A} \end{bmatrix} \in \mathbb{R}^{2n \times 2n}, \\ 
    \label{eq:covMatG} \bm{G} = \begin{bmatrix} \bm{I} & -\KalmanGain \\ \bm{0} & \KalmanGain \end{bmatrix} \in \mathbb{R}^{2n\times (n+p)}, \\ 
    \label{eq:covMatPi} \bm{\Pi} = \bm{G}\, \operatorname{diag}(\bm{Q}, \bm{R}) \bm{G}^\top \in \mathbb{R}^{2n \times 2n}. 
\end{align}
\end{theorem}
\begin{proof}
From the system equations, the one-step evolution of the estimation error $\bm{e}_k := \bm{x}_k - \hat{\bm{x}}_k$ and of the actuation-projected residual $\tilde{\bm{r}}_k := \hat{\bm{x}}_k - \ActuationProjectedState_{k}$ is:
\begin{align}
	\bm{e}_{k+1} &= (\bm{A} - \KalmanGain \bm{C}) \bm{e}_k + \bm{w}_k - \KalmanGain \bm{v}_k, \\
	\tilde{\bm{r}}_{k+1} &= \bm{A} \tilde{\bm{r}}_k + \KalmanGain \bm{C} \bm{e}_k + \KalmanGain \bm{v}_k.
\end{align}
Stacking ${\bm{z}_k := [\bm{e}_k^\top, \tilde{\bm{r}}_k^\top]^\top}$ yields
\begin{align}\label{eq:projAnomDynamics}
	\bm{z}_{k+1} = \bm{F} \bm{z}_k + \bm{G} \bm{\eta}_k,
\end{align}
with $\bm{F}, \bm{G}$ as in \cref{eq:covMatF,eq:covMatG} and $\bm{\eta}_k \triangleq [\bm{w}_k^\top, \bm{v}_k^\top]^\top$.  
The covariance propagates as
\begin{align*}
	\bm{P}_{z,k+1} = & \ \bm{F}\,\mathbb{E}[\bm{z}_k\bm{z}_k^\top]\bm{F}^\top 
	+ \bm{F}\,\mathbb{E}[\bm{z}_k\bm{\eta}_k^\top]\bm{G}^\top \\
	&+ \bm{G}\,\mathbb{E}[\bm{\eta}_k\bm{z}_k^\top]\bm{F}^\top 
	+ \bm{G}\,\mathbb{E}[\bm{\eta}_k\bm{\eta}_k^\top]\bm{G}^\top.
\end{align*}
Since $\bm{\eta}_k$ is white, zero-mean, and independent of $\bm{z}_k$, 
$\mathbb{E}[\bm{z}_k\bm{\eta}_k^\top]=\bm{0}$ and 
$\mathbb{E}[\bm{\eta}_k\bm{\eta}_k^\top]=\operatorname{diag}(\bm{Q},\bm{R})$, giving
\[
	\bm{P}_{z,k+1} = \bm{F} \bm{P}_{z,k} \bm{F}^\top + \bm{\Pi},
\]
with $\bm{\Pi}$ as in \cref{eq:covMatPi}. 

At synchronization $k=0$, $\operatorname{Cov}(\bm{e}_0)=\bm{P}$ and $\tilde{\bm{r}}_0=\bm{0}$, hence
\[
\bm{P}_{z,0} = 
\begin{bmatrix} 
	\bm{P} & \bm{0} \\ \bm{0} & \bm{0}
\end{bmatrix}
= \operatorname{diag}(\bm{P},\bm{0}).
\]
Iterating \eqref{eq:covRecurs} $k$ times yields $\bm{P}_{z,k}$; the desired covariance is the $(2,2)$ block, $\bm{\Sigma}_{\tilde{r},k}$.
\end{proof}

\begin{theorem}[Confidence for the Anomaly Measure]
\label{thm:anomaly_threshold}
Consider the anomaly measure
    \begin{align}\label{eq:anomalyScore}
        \tilde{z}_k = \tilde{\bm{r}}_k^\top \bm{\Sigma}_{\tilde{r}, k}^{-1} \tilde{\bm{r}}_k.
    \end{align}
Choosing $z_{\text{x}} = F_{\chi^{2}(n)}^{-1}(\desprob)$ ensures
$\mathbb{P}(\tilde{z}_k \le z_{\text{x}} \mid \mathcal{H}_0) = \desprob$ for any given probability $\desprob$.
\end{theorem}

\begin{proof}
    Under $\mathcal{H}_0$, the residual $\tilde{\bm{r}}_k$ is a zero-mean Gaussian vector, $\tilde{\bm{r}}_k \sim \mathcal{N}(\bm{0}, \bm{\Sigma}_{\tilde{r}, k})$. 
    The time-varying normalization with the covariance $\bm{\Sigma}_{\tilde{r}, k}$ at each step, computed as in \cref{thm:residual_cov}, ensures that the resulting statistic has a stationary distribution, $\tilde{z}_k \sim \chi^{2}(n)$, for all $k$. The theorem’s result follows from the definition of the inverse CDF, $F_{\chi^{2}(n)}^{-1}$.
\end{proof}

\subsection{Command Scaling Active Defence}
As an active defence mechanism,  we propose to reduce the magnitude of command $\bm{u}^{\text{nom}}_{k}$ of \cref{eq:pid_actuation_joint} through a gain scaling function $\pinExp(\tilde{z})$, with $\tilde{z}$ as in \cref{eq:anomalyScore}.

Let $z_{\text{x}}>0$ be a design abscissa and
$\beta \in (0,1)$ 
the desired gain factor at $\tilde{z}=z_{\text{x}}$. 
Introduce a shape exponent $\gamma>0$ and a smooth, strictly decreasing function ${\pinExp : [0,\infty) \to (0,1]}$,
\begin{align}
	\pinExp(\tilde{z}) &= \exp\!\left[-\left(\dfrac{\tilde{z}}{z_{\text{scale}}}\right)^{\gamma}\right], 
	\label{eq:exp_decreasing_gain_law_zero} \\
z_{\text{scale}} &= \dfrac{z_{\text{x}}}{\big(-\ln (\beta) \big)^{1/\gamma}},
\end{align}
satisfying $\pinExp(0)=1$, $\pinExp(z_{\text{x}})= \beta $, and $\pinExp(\tilde{z})\to 0$ as $\tilde{z}\to\infty$.

\begin{theorem}\label{th:main}
Consider a robotic manipulator with nominal joint-space control input $\bm{u}^{\text{nom}}_k$ as in~\eqref{eq:pid_actuation_joint}. 
Let {${\tilde{z}_k \in [0,\infty)}$} denote the anomaly score defined in~\eqref{eq:anomalyScore}, and let {${f : [0,\infty)\to (0,1]}$} be a strictly decreasing function satisfying
\begin{align*}
f(0) = 1, \quad \lim_{\tilde{z}\to\infty} f(\tilde{z}) = 0.
\end{align*}
Define the scaled control law
\begin{align}\label{eq:finalControl}
\bm{u}_k = f(\tilde{z}_k)\, \bm{u}^{\text{nom}}_k. 
\end{align}
Then the following hold:
\begin{enumerate}

\item \textbf{Probabilistic actuation guarantee.} Let $\tilde{z}_k \sim \chi^{2}(n)$ under the null hypothesis, and let $z_{\text{x}} = F^{-1}_{\chi^{2}(n)}(\desprob)$ denote the inverse CDF at confidence level $\desprob \in (0,1)$. If design parameters of $\pinExp(\cdot)$ are chosen such that
\begin{align*}
f(z_{\text{x}}) \geq \beta,
\end{align*}
for some $\beta \in (0,1)$, then with probability at least $\desprob$, the scaled input magnitude satisfies
\begin{align}\label{eq:probGainTh}
\|\bm{u}_k\| \;\geq\; \beta \,\|\bm{u}^{\text{nom}}_k\|.
\end{align}
\item \textbf{Closed-loop stability (frozen gain).}
Consider the discretized double-integrator dynamics induced by inverse-dynamics compensation in~\eqref{eq:idealDynamics}.
If the nominal frozen-gain closed loop is Schur stable (i.e.,~\eqref{eq:finalControl} with $\pinExp(\cdot)\equiv 1$), then for any constant gain factor $\bar f\in(0,1]$ the frozen-gain closed loop under \cref{eq:finalControl} is Schur stable, and hence exponentially stable.
\end{enumerate}
\end{theorem}

\begin{proof}
\textbf{Probabilistic actuation guarantee.} Under $\mathcal{H}_0$, the anomaly score follows $\tilde{z}_k \sim \chi^{2}(n)$. By definition of the quantile,
{${\Pr(\tilde{z}_k \le z_{\text{x}}) = \desprob}$}.
On this event, since $\pinExp(\cdot)$ is non-increasing, we have $f(\tilde{z}_k) \geq f(z_{\text{x}})$. By assumption, {${f(z_{\text{x}}) \geq \beta}$}.
Substituting into~\eqref{eq:finalControl},
\begin{align*}
\|\bm{u}_k\| = f(\tilde{z}_k)\,\|\bm{u}^{\text{nom}}_k\| \geq \beta \|\bm{u}^{\text{nom}}_k\|.
\end{align*}
Therefore, with probability at least $\desprob$, the commanded input magnitude remains at least a fraction $\beta$ of the nominal value.

\textbf{Closed-loop stability (frozen gain).}
Under inverse-dynamics compensation, and treating the Jacobian pseudoinverse as locally constant, it suffices to consider one joint with
${u_{k,j}^{\mathrm{nom}}=-\bm{K}\bm{x}_{k,j}}$, where ${\bm{x}_{k,j}\in\mathbb{R}^2}$ and $\bm{K}=[K_p\;K_d]$.
For a frozen ${\bar f\in(0,1]}$, the closed-loop matrix is
${\bm{A}_{\mathrm{cl}}(\bar f)=\bm{A}_j-\bar f\,\bm{B}_j\bm{K}}$,
where $\bm{A}_j,\bm{B}_j$ are the joint double-integrator matrices.
Applying the second-order Jury criterion to the characteristic polynomial of $\bm{A}_{\mathrm{cl}}(\bar f)$, Schur stability is equivalent to
${\bar f T_s^2 K_p>0}$,
${2-\bar f T_s K_d>0}$,
and
${K_d>\tfrac{T_s}{2}K_p}$.
These conditions hold at $\bar f=1$ by assumption, hence
${K_p>0}$,
${2>T_sK_d}$,
and
${K_d>\tfrac{T_s}{2}K_p}$.
Therefore, for every $\bar f\in(0,1]$, the first condition holds since $\bar f>0$ and $K_p>0$, the second becomes less restrictive as $\bar f$ decreases, and the third is independent of $\bar f$.
Thus $\bm{A}_{\mathrm{cl}}(\bar f)$ is Schur for all $\bar f\in(0,1]$, and the frozen-gain closed loop is exponentially stable.
\end{proof}

In practice, $f(\tilde{z}_k)$ is time-varying and state-dependent; however, item~1 of the theorem confines $f$ to the compact set $[\beta,1]$ with probability $\desprob$, over which frozen stability holds uniformly.



\begin{corollary}\label{th:corollaryPercentage}
	Fix a desired confidence level $\desprob\in(0,1)$ and a desired gain floor $\beta\in(0,1)$. 
	Fix $z_{\text{x}} = F^{-1}_{\chi^{2}(n)}(\desprob)$ in~\eqref{eq:exp_decreasing_gain_law_zero} so that $f(z_{\text{x}})=\beta$.
	Then, for the control law~\eqref{eq:finalControl} with $\tilde{z}_k$ as in \cref{eq:anomalyScore}, item~1 of \cref{th:main} guarantees that the magnitude of the final command satisfies~\eqref{eq:probGainTh}
	with probability $\desprob$, i.e., the final input remains at least {${ \beta \cdot 100\%}$} of the nominal magnitude with probability $\desprob$.
\end{corollary}

\newcommand{\extract}[3]{\pi_{#2}\!\bigl(#1\bigr)_{#3}}

\section{Optimal Stealth Attack}\label{sec:attack}
We now formalize the adversary's strategy under the assumptions of \cref{subsec:assumptions}.
The attacker's goal is to induce stealthy malicious end-effector accelerations.

To express the predicted trajectories of different vectors (e.g., acceleration, velocity) from a closed-loop simulation, we introduce the following compact notation. 

Let $\simul_{k,j}(\bm a)$ denote the \emph{$j$-step-ahead prediction at time $k{+}j$,
computed at time $k$}, of the closed-loop system under the attack sequence $[\bm a,\bm 0,\dots,\bm 0]$
(of length $j+1$).

For a generic vector $\bm{v}$, we define
\begin{equation}
    \bm{v}_{k+j}^{\text{SIM}} :=
    \pi_{\bm{v}}\!\bigl(\simul_{k,j}(\bm{a})\bigr),
    \label{eq:projection_definition}
\end{equation}
where $\pi_{\bm{v}}(\cdot)$ extracts $\bm{v}$ from the noise-free simulated trajectory at prediction step $j$
(i.e., the terminal value at time $k+j$).
The operator $\simul_{k,j}(\bm{a})$ models Assumption 1 of \cref{subsec:assumptions}, as it propagates \cref{eq:lti_plant,eq:kalmanX,eq:kalmanR} under perfect-model assumptions (exact initial state and noise-free dynamics) and it computes the control input via \cref{eq:finalControl} including the proposed active defence mechanism.



\subsection{Attack Model and Delay Structure}
Due to estimator and control delays, an injected signal at time $k$, $\bm{a}_k$, influences the end-effector acceleration two steps later (under the closed-loop dynamics in \cref{eq:lti_plant,eq:kalmanX,eq:kalmanR}). Specifically,
\begin{equation}\label{eq:simulatedAcceleration}
    \PredictedAcceleration_{k+2} = 
    \pi_{\ddot{\bm{p}}}\!\bigl(\simul_{k,2}(\bm{a}_{k})\bigr),
\end{equation}
where $\simul_{k,2}$ denotes a two-step-ahead simulation of the closed-loop system under attack sequence $[\bm{a}_{k}, \bm{0}, \bm{0}]$.

The attacker's high-level objective at each time step is to make the predicted end-effector acceleration match a desired target acceleration $\TargetAccAttack_{k+2}$. This is formulated as:
\begin{equation}
    \label{eq:objFunIntro}
    \min_{\bm{a}_k} \quad \tfrac{1}{2} \left\| \TargetAccAttack_{k+2} - \pi_{\ddot{\bm{p}}}\!\bigl(\simul_{k,2}(\bm{a}_{k})\bigr) \right\|^2.
\end{equation}

\subsection{Incremental Attack Formulation}\label{sec:IncrementalGradient}
Due to feedback linearization, the manipulator’s joint-space dynamics reduce to double integrators, which entails an \emph{integrator vulnerability}: a persistent sensor bias can induce drift in the regulated variables. This vulnerability has been analyzed for sensor \emph{bias injection attacks} (BIAs)—i.e., constant sensor injections—in linear systems~\cite{tosun2025kullback}. 
Although the joint–task mapping is nonlinear and we consider more general \emph{false data injection attacks} (FDIAs) with arbitrary sensor injections, we argue that effective FDIAs still exploit the integrator vulnerability \emph{locally}.
For this reason, we model the attack \emph{incrementally}:
\begin{equation}\label{eq:incremental}
    \bm{a}_k \;=\; \bm{a}_{k-1} + \bm{\Delta}_k, \qquad k\ge 1,
\end{equation}
where $\bm{\Delta}_k$ is an increment and $\bm{a}_k$ is initialized as $\bm{a}_0 = \bm{\Delta}_0$ at the attack onset. 
This parameterizes an FDIA as deviations from a baseline BIA (via the increments $\bm{\Delta}_k$), and is convenient for gradient-based synthesis.\\

A first-order expansion of $\pi_{\ddot{\bm{p}}}\!\bigl(\simul_{k,2}(\bm{a})\bigr)$ around $\bm{a}_{k-1}$, with $\PredictedAcceleration_{k+2}$ as in \cref{eq:simulatedAcceleration}, gives:
\begin{equation}\label{eq:nonLinearMap2}
    \PredictedAcceleration_{k+2} \;\approx\;
    \pi_{\ddot{\bm{p}}}\!\bigl(\simul_{k,2}(\bm{a}_{k-1})\bigr)
    + \simulLin_k \bm{\Delta}_k \;+\; \mathcal{O}(\|\bm{\Delta}_k\|^2),
\end{equation}
where the Jacobian
\begin{equation}\label{eq:linearizedMap}
    \simulLin_{k} \;=\; \left. \frac{\partial}{\partial \bm{a}}\,
    \pi_{\ddot{\bm{p}}}\!\bigl(\simul_{k,2}(\bm{a})\bigr) \right|_{\bm{a} = \bm{a}_{k-1}}
\end{equation}
is computed numerically via a central-difference scheme with step-size tuning~\cite{ramachandran_sampling_2021}; $\simulLin_k$ provides the local sensitivity needed to differentiate the objective \cref{eq:objFunIntro}.

By substituting the linear attack model \eqref{eq:nonLinearMap2} into \eqref{eq:objFunIntro} and introducing a regularization term, a quadratic objective function is obtained. Specifically, 
the attack increment $\bm{\Delta}_k$ minimizes the quadratic cost
\begin{equation}\label{eq:quadCostAttack}
    \tfrac{1}{2}\big\| 
        \simulLin_k \bm{\Delta}_k
        - \TargetAccAttack_{k+2}
        + \pi_{\ddot{\bm{p}}}\!\bigl(\simul_{k,2}(\bm{a}_{k-1})\bigr)
    \big\|^2
    + \tfrac{\optLambda}{2}\|\bm{\Delta}_k\|^2,
\end{equation}
where $\optLambda > 0$ penalizes large increments.

\subsection{Attack Objective}\label{sec:attackerPD}
We restrict the attack to the translational DOFs, as these are most relevant to adversarial manipulation.
The adversary defines $\TargetAccAttack_{k+2}$ of \cref{eq:objFunIntro} via a one-step-ahead PD law plus feedforward based on a plan $\{ \TargetPositionAttack_{k}, \TargetVelocityAttack_{k}\}$.
Specifically, the PD errors are computed as the difference between the desired and predicted quantities, under the incremental attack approach ($\bm{\Delta}_k{=}0$ or equivalently, $\bm{a}_k = \bm{a}_{k-1}$):
\begin{align}\label{eq:PDattacker}
    \TargetAccAttack_{k+2} &= \GainAttackProportional
    \big(\TargetPositionAttack_{k+1} - \PredictedPosition_{k+1}\big) \nonumber
    + \GainAttackDerivative
    \big(\TargetVelocityAttack_{k+1} - \PredictedVelocity_{k+1}\big),
\end{align}
where
\begin{align*}
    \PredictedPosition_{k+1} &= 
    \pi_{\bm{p}}\!\bigl(\simul_{k,1}(\bm{a}_{k-1})\bigr), \\
    \PredictedVelocity_{k+1} &=
    \pi_{\dot{\bm{p}}}\!\bigl(\simul_{k,1}(\bm{a}_{k-1})\bigr),
\end{align*}
are the one-step-ahead end-effector position and velocity obtained from the closed-loop simulator under the attack sequence $[\bm{a}_{k-1}, \bm{0}]$.
The resulting $\TargetAccAttack_{k+2}$ is therefore a way to counteract the predicted drift resulting from $\bm{\Delta}_k{=}0$.

\subsection{Stealth Constraint}
The ADS residual is modeled as
\begin{equation}
    \Residual_k = \bm{\Delta}_k + \baselineIn_k,
\end{equation}
with $\baselineIn_k$ denoting the baseline innovation, defined as
\begin{equation}
    \baselineIn_k = \bigl(\bm{y}_k + \bm{a}_{k-1} - \hat{\bm{y}}_k\bigr).
\end{equation}

The adversary uniformly allocates a per-step anomaly budget
\begin{equation}
    \tau' = \tfrac{\ThresholdVal}{\AttackDuration},
    \label{eq:tauQuota}
\end{equation}
so that the stealth constraint becomes
\begin{equation}\label{eq:stealthConstrUnexpand}
    \bigl(\bm{\Delta}_k + \baselineIn_k\bigr)^\top 
    \bm{\Sigma}^{-1}
    \bigl(\bm{\Delta}_k + \baselineIn_k\bigr) \le \tau'.
\end{equation}

\subsection{QCQP Formulation}
Considering the objective function \cref{eq:quadCostAttack} and expanding the constraint of \cref{eq:stealthConstrUnexpand}, the adversary's problem reduces to the convex QCQP
\begin{equation}
\begin{aligned}
    \min_{\bm{\Delta}_k}\ & \tfrac{1}{2}\bm{\Delta}_k^\top \bm{H}\bm{\Delta}_k + \bm{g}^\top \bm{\Delta}_k \\
    \text{s.t.}\ & \bm{\Delta}_k^\top \optimAmat\bm{\Delta}_k + \bm{b}^\top \bm{\Delta}_k + c \le 0,
\end{aligned}
\label{eq:optimProb}
\end{equation}
with parameters
\begin{align*}
    \bm{H} &= \simulLin_k^\top \simulLin_k + \optLambda \bm{I}, \\
    \bm{g} &= -\simulLin_k^\top \Big(\TargetAccAttack_{k+2} - \pi_{\ddot{\bm{p}}}\!\bigl(\simul_{k,2}(\bm{a}_{k-1})\bigr)\Big), \\
    \optimAmat &= \bm{\Sigma}^{-1}, \\ 
    \bm{b} &= 2\bm{\Sigma}^{-1}\baselineIn_k, \\
    c &= \baselineIn_k^\top \bm{\Sigma}^{-1}\baselineIn_k - \tau'.
\end{align*}


\begin{theorem}[Convexity of adversary's QCQP]
Problem \eqref{eq:optimProb} is convex, and thus admits the global minimizer $\bm{\Delta}_k^\star$.
\end{theorem}

\begin{proof}
The cost Hessian $\bm{H} = \simulLin_k^\top \simulLin_k + \optLambda \bm{I} \succeq \optLambda \bm{I} \succ 0$, so the objective is strictly convex.  
The constraint has Hessian $\optimAmat = \bm{\Sigma}^{-1} \succ 0$, yielding a convex (ellipsoidal) feasible set.
A strictly convex objective over a non-empty convex feasible set guarantees existence and uniqueness of the global minimizer $\bm{\Delta}_k^\star$.
\end{proof}

\begin{remark}[Feasibility]
The feasible set is non-empty at every step: the choice $\bm{\Delta}_k = -\baselineIn_k$ drives the residual to zero and trivially satisfies the constraint for any $\tau' > 0$.
\end{remark}

\subsection{Iterative Attack Construction}
The attack evolves incrementally as in \cref{eq:incremental}
\begin{equation}
    \bm{a}_k = \bm{a}_{k-1} + \bm{\Delta}_k^\star,
\end{equation}
initialized with $\bm{a}_0 = \bm{\Delta}_0^\star$.  
At each step, the adversary runs an internal simulation to compute $\simul_k$ and $\simulLin_k$, then solves \eqref{eq:optimProb} to determine $\bm{\Delta}_k^\star$.

\section{Simulation Results}\label{sec:simul}
For reproducibility, the Matlab implementation used in this work has been archived and is publicly available \cite{gualandi_2026_zenodo_18813338}.

\begin{figure}[tb]
\centering
\vspace{0.70em}
\includegraphics[width=\columnwidth]{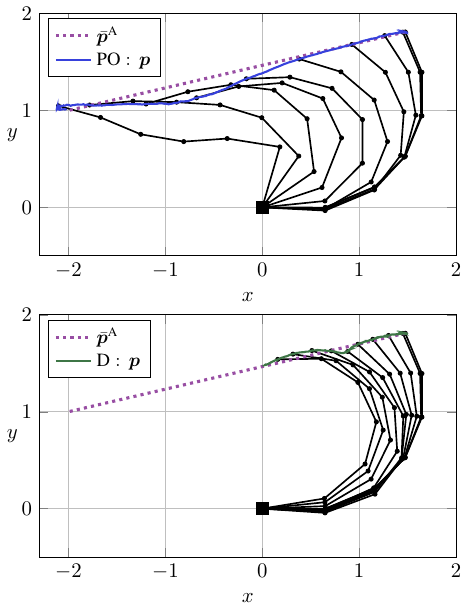}
\caption{
End-effector trajectories: attacker's reference $\TargetPositionAttack$, defended only by the passive detector (PO) and by the proposed active defence (D).
The latter defence limits drift toward the malicious target.}
\label{fig:EXP1robot}
\end{figure}

\begin{figure}[tbh]
\vspace*{1.5ex} 
\centering
\resizebox{\columnwidth}{!}{
\includegraphics{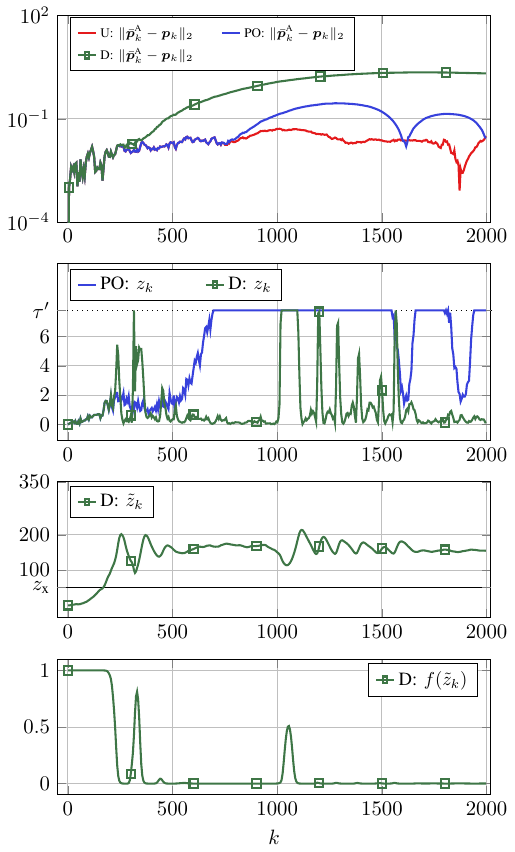}
}
\caption{
Attack tracking error $\|\TargetPositionAttack_k - \bm{p}_k\|_2$, Mahalanobis distances $z_k$ and $\tilde{z}_k$ of \cref{eq:mahalanobis_distance,eq:anomalyScore},  and scaling $\pinExp(\tilde{z})$ of \cref{eq:exp_decreasing_gain_law_zero}. With active defence (D), $\tilde{z}_k$ increases until scaling becomes effective, then settles, limiting attacker-realizable accelerations.}
\label{fig:EXP1signals}
\end{figure}

\begin{figure*}[!ht]
\centering
\resizebox{\textwidth}{!}{
\includegraphics{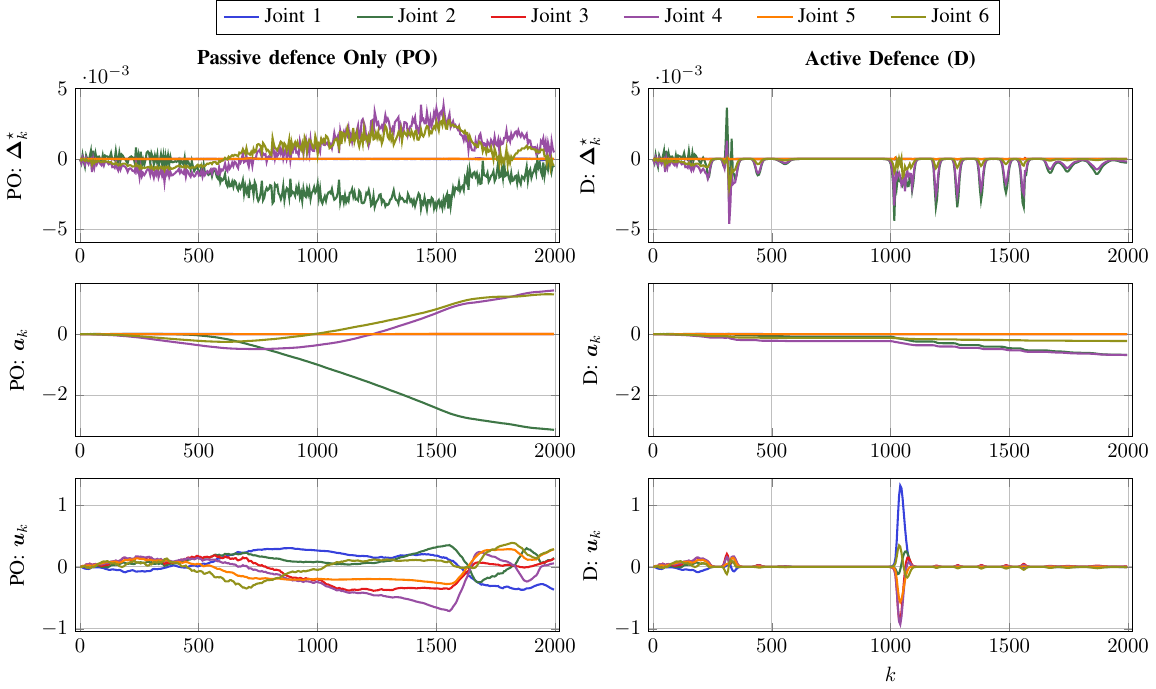}
}
\caption{
{Adversary and control signals.}
Optimal increment $\bm{\Delta}^\star_k$ of \cref{eq:optimProb}, attack sequence $\bm{a}_k$ of \cref{eq:incremental}, and resulting actuation $\bm{u}_k$ of \cref{eq:pid_actuation_task,eq:finalControl}.
Under the proposed active defence (D), the attacker escalates $\bm{\Delta}^\star_k$ but scaling suppresses the realized $\bm{u}_k$, constraining task-space manipulation.}
\label{fig:EXP1signals2}
\end{figure*}
\subsection{Experimental Setup}
The plant is a 6-DOF planar manipulator with link lengths $[0.65,\,0.55,\,0.45,\,0.45,\,0.45,\,0.45]$~m.  
The task space comprises planar position $(x,y)$ and one orientation $\theta_z$.  
The noise covariances are $\bm{R} = \csvConstNoDefeYesAtt{senseCovBlk}{1} \bm{I}$
and
{${
\bm{Q} = \mathrm{blkdiag}_j\!\big(q_c \,\bm{Q}_{\text{base}}\big)
}$}, where 
{${\bm{Q}_{\text{base}} = [ T_s^3/3, T_s^2/2; T_s^2/2, T_s]}$}
and {${q_c = 10^{-2}\,\mathrm{rad}^2/\mathrm{s}^3}$}, determined with the tuning-knob process.

The nominal task is to maintain the fixed end-effector pose with zero reference velocity and acceleration
\begin{align*}
\TargetPosition_0 =
\begin{bmatrix}\csvConstNoDefeYesAtt{pInit}{1} \\ \csvConstNoDefeYesAtt{pInit}{2} \end{bmatrix}\text{ m}, 
\qquad 
\TargetRotationMatrix_0 \text{ from }
\bm{q}_0=[0, \tfrac{\pi}{8},\cdots,\tfrac{\pi}{8}]^{\top}.
\end{align*}

The system is discretized with sampling time {${T_s= \csvConstNoDefeYesAtt{Ts}{1}}$}s.  
The passive defence ($\chi^2$) uses a sliding window of $W = \csvConstYesDefYesAtt{detectionWindow}{1}$ samples, and it is calibrated for an Average Run Length (ARL) of $\csvConstYesDefYesAtt{desiredChiSquare}{1}$ samples (i.e., one false positive every $0.95$ years under $\mathcal{H}_0$), which is realized through \cref{eq:desired_arl} by setting $\tau = \csvConstYesDefYesAtt{ADStau}{1}$. 

The active defence uses $\pinExp(\cdot)$ of \cref{eq:exp_decreasing_gain_law_zero} with {$ { \beta = (1 - 10^{-3} )}$},
{${\desprob = (1 - \csvConstYesDefYesAtt{scalingOneMinusP}{1})}$},
and $\gamma = \csvConstYesDefYesAtt{scalingExpSteepness}{1}$,
determining $z_{\text{x}} = \csvConstYesDefYesAtt{scalingZx}{1}$.
For \cref{th:corollaryPercentage}, we have formal guarantees that, on average and under $\mathcal{H}_0$, a single sample of the control command has magnitude below
$\csvConstYesDefYesAtt{scalingExcessPerc}{1}\%$
of the nominal one every $\csvConstYesDefYesAtt{scalingExcessHours}{1}$ hours, which we consider negligible.

The defence adds negligible computational overhead.
The QCQP for the attacker (not part of the defence) solves in 40 ms/step with a standard solver (Gurobi).

\paragraph*{Attacker.}
The attacker's goal is to displace the end-effector from $\TargetPositionAttack_0 = \TargetPosition_0$ to
$\TargetPositionAttack_{(\AttackDuration-1)} = 
[-2,1]^\top$. This maneuver is interpolated by quintic polynomials with zero boundary velocities and accelerations, generating the reference trajectory $\{\TargetPositionAttack_k, \TargetVelocityAttack_k, \PlannedAccAttack_{k}\}$ for $k \in \{0,\ldots,\AttackDuration-1\}$.  
The anomaly budget is distributed uniformly across time, resulting in a per-step constraint $\tau'=\tau/\AttackDuration=\csvConstYesDefYesAtt{ADStauPrime}{1}$.

\paragraph*{Evaluation Metrics.}
Let $\{\bm{u}_k\},\; k \in \{0,\ldots,T-1\}$ denote a sequence of control commands. The following metric quantifies the control effort:
\begin{align*}
\mean{\{\bm{u}_k\}} 
&\coloneqq
\frac{1}{T}
\sum_{k}
\bigl\lVert \bm{u}_k \bigr\rVert_2.
\end{align*}

Let
$\{\bar{\bm{p}}_k\},\; k \in \{0,\ldots,T-1\}$ denote a reference hand position trajectory, and
$\{\bm{p}_k\}$ the realized one. The following metrics quantify a task deviation:
\begin{align*}
\devmax{\bar{\bm{p}}_k}{\bm{p}_k} 
&\coloneqq
\max_{k}
\bigl\lVert \bar{\bm{p}}_k - \bm{p}_k \bigr\rVert_2, \\
\devRMS{\bar{\bm{p}}_k}{\bm{p}_k} 
&\coloneqq
\left(
  \frac{1}{\AttackDuration}
  \sum_{k}
  \bigl\lVert \bar{\bm{p}}_k - \bm{p}_k \bigr\rVert_2^2
\right)^{1/2}.
\end{align*}

\subsection{Main Results}

\paragraph*{Undefended (U).}
When no active or passive defence is in place, the end-effector closely follows the malicious trajectory with a small maximum and RMS deviation from the reference (see first column of \cref{tab:metrics}).

\paragraph*{Passive Defence Only (PO).}
In the presence of the $\chi^2$ ADS using anomaly measure $z_k$ as in \cref{eq:mahalanobis_distance}, the attacker needs to satisfy the constraint of \cref{eq:stealthConstrUnexpand}. The resulting task deviation mildly increases compared to the Undefended case, and the attacker can still reach its final position goal (see the upper diagrams of \cref{fig:EXP1robot} and \cref{fig:EXP1signals}, and the second column of \cref{tab:metrics}).
Throughout, the anomaly measure $z_k$ remains below $\tau'$, so the passive ADS never fires alarms (see second diagram of \cref{fig:EXP1signals}).
\paragraph*{With the proposed active defence (D).} With the command scaling of \cref{eq:finalControl} in place,
in the initial phase ($k<250$) the optimal attack increment $\bm{\Delta}_k^\star$ and the deviation $\|\TargetPositionAttack_k - \bm{p}_k\|_2$ mirror PO and U, as $\tilde{z}_k$ is small and $\pinExp(\cdot)\approx1$ (see \cref{fig:EXP1signals,fig:EXP1signals2}).
Subsequently, $\tilde{z}_k$ increases, determining an increasing scaling of command $\bm{u}_k$ (see \cref{fig:EXP1signals2}).
This incentivizes the attacker to increase injections; $z_k$ quickly climbs to the per-step limit $\tau'$ rendering the constraint of \cref{eq:stealthConstrUnexpand} active for a short time. 

As the attack progresses, gain reduction $\pinExp(\tilde{z}_k)$ disincentivizes the attacker policy of \cref{sec:attackerPD}, since increased anomaly would further strengthen scaling.
Consequently, $\tilde{z}_k$ tends to settle to an asymptotic value at which the acceleration commands are greatly diminished in magnitude.
Relative to PO, the realized accelerations are significantly attenuated, yielding a much smaller mean control effort (see \cref{tab:metrics}).

\subsection{Comparative Analysis and Discussion}

\begin{table}[t]
\centering
\caption{Task deviation and control effort metrics}
\label{tab:metrics}
\setlength{\tabcolsep}{4.5pt}
\begin{tabular}{lccc}
\toprule
Metric & \makecell{U \\ (no defence)} & \makecell{PO\\($\chi^2$ only)} & \makecell{D\\(active defence)}  \\
\midrule
$\devmax{\TargetPositionAttack}{\bm{p}}$&
$\csvNoDefeNoADSYesAtt{attackerDeviation_HandP_l2_max}{1}$ &
$\csvNoDefeYesAtt{attackerDeviation_HandP_l2_max}{1}$ &
$\csvYesDefYesAtt{attackerDeviation_HandP_l2_max}{1}$ 
\\
$\devmax{\TargetPosition}{\bm{p}}$ &
$\csvNoDefeNoADSYesAtt{deviation_HandP_l2_max}{1}$ &
$\csvNoDefeYesAtt{deviation_HandP_l2_max}{1}$ &
$\csvYesDefYesAtt{deviation_HandP_l2_max}{1}$ 
\\ 
$\devRMS{\TargetPositionAttack}{\bm{p}}$&
$\csvNoDefeNoADSYesAtt{att_rmsError}{1}$ &
$\csvNoDefeYesAtt{att_rmsError}{1}$ &
$\csvYesDefYesAtt{att_rmsError}{1}$ 
\\
$\devRMS{\TargetPosition}{\bm{p}}$ &
$\csvNoDefeNoADSYesAtt{hand_rmsError}{1}$ &
$\csvNoDefeYesAtt{hand_rmsError}{1}$ &
$\csvYesDefYesAtt{hand_rmsError}{1}$ 
\\ 
$\mean{\{\bm{u}_k\}} $  &
$\csvNoDefeNoADSYesAtt{u_final_l2_avg}{1}$ &
$\csvNoDefeYesAtt{u_final_l2_avg}{1}$ &
$\csvYesDefYesAtt{u_final_l2_avg}{1}$ 
\\
\bottomrule
\end{tabular}
\vspace{-0.5em}
\end{table}

Without defence, or with passive defence only, stealthy FDIAs can precisely steer the end-effector while remaining within the ADS budget; the passive detector is therefore ineffective in isolation.

With the proposed active defence, the attacker faces the following trade-off: increasing injection raises the proposed anomaly-aware command scaling, while reducing injection makes the stealthy attack less effective.
Overall, the attack decreases in effectiveness and the proposed anomaly-aware score $\tilde{z}_k$ settles to an asymptotic value. This plateau represents the sub-optimal injection magnitude the attacker is forced to commit to, akin to an equilibrium in a Stackelberg game.
\cref{fig:EXP1signals} explains the closed-loop interaction:
(i) the attacker increases $\bm{\Delta}^\star_k$ to counter scaling (\cref{fig:EXP1signals2}, top), 
(ii) this raises $z_k$ until the per-step constraint becomes active, 
(iii) meanwhile $\tilde{z}_k$ (based on the actuation-projected predictor) tracks the accumulating state discrepancy and drives $\pinExp(\tilde{z}_k)$ down, 
(iv) the realized $\bm{u}_k$ is therefore attenuated, limiting achievable task-space acceleration and displacement.


An important consideration regards the safety of the proposed active defence in the absence of attacks. 
By selecting the design confidence $\desprob$ and gain floor $\beta$ via \cref{th:corollaryPercentage}, unwarranted scaling under $\mathcal{H}_0$ can be made arbitrarily rare, although statistical fluctuations can in principle still trigger noticeable attenuation.
However, because the gain scaling acts on the command $\bm{u}_k$ prior to the inverse-dynamics compensation, it does not alter gravity or Coriolis compensation.
Finally, extending the stability analysis to the time-varying case of $f(\tilde{z}_k)$ is left for future work; promising directions include (i) low-pass filtering $f(\tilde{z}_k)$, (ii) enforcing a probabilistic lower bound $f\ge\beta$, and (iii) invoking slowly-varying stability results.


\section{Conclusion}\label{sec:conclusion}
This paper addressed the resilience of robotic manipulators against stealthy false data injection attacks (FDIAs).  
We showed that feedback linearization induces an \emph{integrator vulnerability} that allows persistent sensor corruption to remain undetected by $\chi^2$ anomaly detectors while driving the end-effector off-task.  
To counter this threat, we introduced \emph{anomaly-aware command scaling}, a lightweight modification of the control law that attenuates inputs as a function of an anomaly score derived from a measurement-free, actuation-projected predictor. Our analysis established two key guarantees: (i) probabilistic bounds on actuation loss in nominal operation, enabling minimally invasive deployment, and (ii) preservation of closed-loop stability under bounded attenuation.  

On the adversary side, we derived a convex QCQP formulation of the one-step optimal stealthy attack that incorporates the defender's policy, yielding a Stackelberg-type game-theoretic benchmark against which to evaluate defences.
Simulation results on a 6-DOF manipulator confirmed that the proposed defence substantially reduces attacker-induced task-space deviations while maintaining nominal tracking performance in the absence of attacks.

Overall, the proposed approach illustrates how active modification of the control law can transform anomaly detection into a practical resilience mechanism for cybersecure robotic manipulation.

\bibliographystyle{IEEEtran}
\bibliography{references}
\addcontentsline{toc}{section}{References}

 \appendix
  \section{Derivation of the optimal stealthy attack QCQP}\label{app:qcqp}
 \subsection*{Derivation of the Quadratic Objective}
This appendix derives the quadratic cost in \cref{eq:quadCostAttack} from the attack objective \cref{eq:objFunIntro}, the incremental parametrization \cref{eq:incremental}, and the first-order approximation \cref{eq:nonLinearMap2,eq:linearizedMap}.

Starting from \cref{eq:objFunIntro}, the attacker solves
\[
    \min_{\bm{a}_k}\ \tfrac{1}{2}
    \left\|
        \TargetAccAttack_{k+2}
        - \pi_{\ddot{\bm{p}}}\!\bigl(\simul_{k,2}(\bm{a}_k)\bigr)
    \right\|^2.
\]
Using the incremental attack model \cref{eq:incremental},
\[
    \bm{a}_k = \bm{a}_{k-1} + \bm{\Delta}_k,
\]
and the local linearization \cref{eq:nonLinearMap2},
\[
    \pi_{\ddot{\bm{p}}}\!\bigl(\simul_{k,2}(\bm{a}_k)\bigr)
    \approx
    \pi_{\ddot{\bm{p}}}\!\bigl(\simul_{k,2}(\bm{a}_{k-1})\bigr)
    + \simulLin_k \bm{\Delta}_k,
\]
the objective becomes
\begin{align*}
    \tfrac{1}{2}
    \left\|
        \TargetAccAttack_{k+2}
        - \pi_{\ddot{\bm{p}}}\!\bigl(\simul_{k,2}(\bm{a}_{k-1})\bigr)
        - \simulLin_k \bm{\Delta}_k
    \right\|^2.
\end{align*}
Define the baseline tracking error
\[
    \bm{e}_k
    \coloneqq
    \TargetAccAttack_{k+2}
    - \pi_{\ddot{\bm{p}}}\!\bigl(\simul_{k,2}(\bm{a}_{k-1})\bigr),
\]
which is the acceleration mismatch obtained when the attack is held constant, i.e., when $\bm{\Delta}_k=\bm{0}$. Then
\[
    \tfrac{1}{2}
    \left\|
        \TargetAccAttack_{k+2}
        - \pi_{\ddot{\bm{p}}}\!\bigl(\simul_{k,2}(\bm{a}_{k-1})\bigr)
        - \simulLin_k \bm{\Delta}_k
    \right\|^2
    =
    \tfrac{1}{2}\|\simulLin_k \bm{\Delta}_k - \bm{e}_k\|^2,
\]
where the identity $\|-\bm{v}\|^2=\|\bm{v}\|^2$ has been used.

Expanding the squared norm gives
\begin{align*}
    \tfrac{1}{2}\|\simulLin_k \bm{\Delta}_k - \bm{e}_k\|^2
    &=
    \tfrac{1}{2}
    (\simulLin_k \bm{\Delta}_k - \bm{e}_k)^\top
    (\simulLin_k \bm{\Delta}_k - \bm{e}_k) \\
    &=
    \tfrac{1}{2}\bm{\Delta}_k^\top \simulLin_k^\top \simulLin_k \bm{\Delta}_k
    - \bm{e}_k^\top \simulLin_k \bm{\Delta}_k
    + \tfrac{1}{2}\bm{e}_k^\top \bm{e}_k.
\end{align*}
Adding the Tikhonov regularization term introduced in \cref{eq:quadCostAttack} yields
\begin{align*}
    J(\bm{\Delta}_k)
    &=
    \tfrac{1}{2}\bm{\Delta}_k^\top \simulLin_k^\top \simulLin_k \bm{\Delta}_k
    - \bm{e}_k^\top \simulLin_k \bm{\Delta}_k
    + \tfrac{1}{2}\bm{e}_k^\top \bm{e}_k
    + \tfrac{\optLambda}{2}\bm{\Delta}_k^\top \bm{I}\bm{\Delta}_k \\
    &=
    \tfrac{1}{2}\bm{\Delta}_k^\top
    \bigl(\simulLin_k^\top \simulLin_k + \optLambda \bm{I}\bigr)
    \bm{\Delta}_k
    - \bm{e}_k^\top \simulLin_k \bm{\Delta}_k
    + \tfrac{1}{2}\bm{e}_k^\top \bm{e}_k.
\end{align*}
Since $\bm{e}_k^\top \simulLin_k \bm{\Delta}_k$ is a scalar, it is equal to $\bm{\Delta}_k^\top \simulLin_k^\top \bm{e}_k$. Therefore, after discarding the constant term $\tfrac{1}{2}\bm{e}_k^\top \bm{e}_k$, which does not affect the minimizer, the objective takes the standard quadratic form used in \cref{eq:optimProb}:
\[
    \min_{\bm{\Delta}_k}
    \tfrac{1}{2}\bm{\Delta}_k^\top \bm{H}\bm{\Delta}_k
    + \bm{g}^\top \bm{\Delta}_k,
\]
with
\begin{align*}
    \bm{H} &= \simulLin_k^\top \simulLin_k + \optLambda \bm{I}, \\
    \bm{g} &= -\simulLin_k^\top \bm{e}_k \\
           &= -\simulLin_k^\top
           \Bigl(
               \TargetAccAttack_{k+2}
               - \pi_{\ddot{\bm{p}}}\!\bigl(\simul_{k,2}(\bm{a}_{k-1})\bigr)
           \Bigr).
\end{align*}
These are the coefficients used in \cref{eq:optimProb}, and they are equivalent to \cref{eq:quadCostAttack} up to the discarded additive constant.

\subsection*{Derivation of the Stealth Constraint}
The stealth condition in \cref{eq:stealthConstrUnexpand} is
\[
    \bigl(\bm{\Delta}_k + \baselineIn_k\bigr)^\top
    \bm{\Sigma}^{-1}
    \bigl(\bm{\Delta}_k + \baselineIn_k\bigr)
    \le \tau',
\]
where $\tau'$ is the per-step budget defined in \cref{eq:tauQuota}. Since $\bm{\Sigma}\succ0$, the matrix $\bm{\Sigma}^{-1}$ is symmetric, and the quadratic expression expands as
\begin{align*}
    \bigl(\bm{\Delta}_k + \baselineIn_k\bigr)^\top
    \bm{\Sigma}^{-1}
    \bigl(\bm{\Delta}_k + \baselineIn_k\bigr)
    &=
    \bm{\Delta}_k^\top \bm{\Sigma}^{-1} \bm{\Delta}_k
    + \bm{\Delta}_k^\top \bm{\Sigma}^{-1}\baselineIn_k \\
    &\quad
    + \baselineIn_k^\top \bm{\Sigma}^{-1}\bm{\Delta}_k
    + \baselineIn_k^\top \bm{\Sigma}^{-1}\baselineIn_k \\
    &=
    \bm{\Delta}_k^\top \bm{\Sigma}^{-1} \bm{\Delta}_k
    + 2\baselineIn_k^\top \bm{\Sigma}^{-1}\bm{\Delta}_k \\
    &\quad + \baselineIn_k^\top \bm{\Sigma}^{-1}\baselineIn_k.
\end{align*}
Rearranging all terms gives
\[
    \bm{\Delta}_k^\top \bm{\Sigma}^{-1} \bm{\Delta}_k
    + 2\baselineIn_k^\top \bm{\Sigma}^{-1}\bm{\Delta}_k
    + \baselineIn_k^\top \bm{\Sigma}^{-1}\baselineIn_k
    - \tau'
    \le 0.
\]
This is the quadratic inequality in \cref{eq:optimProb} with
\begin{align*}
    \optimAmat &= \bm{\Sigma}^{-1}, \\
    \bm{b} &= 2\bm{\Sigma}^{-1}\baselineIn_k, \\
    c &= \baselineIn_k^\top \bm{\Sigma}^{-1}\baselineIn_k - \tau'.
\end{align*}
Hence, combining the quadratic objective above with the expanded stealth condition yields the convex QCQP  of \cref{eq:optimProb}.

  \section{Derivation of the actuation-projected residual covariance}\label{app:projectedCOV}
  \subsection*{Derivation of the Actuation-Projected Residual Covariance}
This appendix provides the detailed proof of \cref{thm:residual_cov}. The derivation starts from the plant and estimator dynamics in \cref{eq:lti_plant,eq:kalmanX,eq:kalmanR}, the synchronization condition \cref{eq:resync}, the actuation-projected predictor \cref{eq:actuation_projected_state}, and the residual definition \cref{eq:residualProj}.

Under the nominal hypothesis $\mathcal{H}_0$, there is no attack, so $\tilde{\bm{y}}_k=\bm{y}_k$ in \cref{eq:measurement_model}. We analyze the covariance over a horizon of $k$ steps after the most recent synchronization, which is taken without loss of generality at time $0$.

\subsection*{Error Dynamics}
Define the Kalman estimation error as
\[
    \bm{e}_k \coloneqq \bm{x}_k - \hat{\bm{x}}_k.
\]
Using \cref{eq:lti_plant,eq:kalmanX,eq:kalmanR}, its one-step evolution is
\begin{align*}
    \bm{e}_{k+1}
    &= \bm{x}_{k+1} - \hat{\bm{x}}_{k+1} \\
    &= (\bm{A}\bm{x}_k + \bm{B}\bm{u}_k + \bm{w}_k)
      - (\bm{A}\hat{\bm{x}}_k + \bm{B}\bm{u}_k + \KalmanGain \bm{r}_k) \\
    &= \bm{A}(\bm{x}_k - \hat{\bm{x}}_k) + \bm{w}_k - \KalmanGain \bm{r}_k \\
    &= \bm{A}\bm{e}_k + \bm{w}_k - \KalmanGain(\bm{y}_k - \bm{C}\hat{\bm{x}}_k) \\
    &= \bm{A}\bm{e}_k + \bm{w}_k - \KalmanGain(\bm{C}\bm{x}_k + \bm{v}_k - \bm{C}\hat{\bm{x}}_k) \\
    &= (\bm{A} - \KalmanGain\bm{C})\bm{e}_k + \bm{w}_k - \KalmanGain\bm{v}_k.
\end{align*}

Next, recall from \cref{eq:residualProj} that the actuation-projected residual is
\[
    \tilde{\bm{r}}_k = \hat{\bm{x}}_k - \ActuationProjectedState_k.
\]
Combining \cref{eq:kalmanX,eq:actuation_projected_state} gives
\begin{align*}
    \tilde{\bm{r}}_{k+1}
    &= \hat{\bm{x}}_{k+1} - \ActuationProjectedState_{k+1} \\
    &= (\bm{A}\hat{\bm{x}}_k + \bm{B}\bm{u}_k + \KalmanGain \bm{r}_k)
      - (\bm{A}\ActuationProjectedState_k + \bm{B}\bm{u}_k) \\
    &= \bm{A}(\hat{\bm{x}}_k - \ActuationProjectedState_k) + \KalmanGain \bm{r}_k \\
    &= \bm{A}\tilde{\bm{r}}_k + \KalmanGain(\bm{y}_k - \bm{C}\hat{\bm{x}}_k) \\
    &= \bm{A}\tilde{\bm{r}}_k + \KalmanGain(\bm{C}\bm{x}_k + \bm{v}_k - \bm{C}\hat{\bm{x}}_k) \\
    &= \bm{A}\tilde{\bm{r}}_k + \KalmanGain\bm{C}\bm{e}_k + \KalmanGain\bm{v}_k.
\end{align*}

\subsection*{Augmented Linear System}
Introduce the augmented state
\[
    \bm{z}_k \coloneqq
    \begin{bmatrix}
        \bm{e}_k \\
        \tilde{\bm{r}}_k
    \end{bmatrix},
    \qquad
    \bm{\eta}_k \coloneqq
    \begin{bmatrix}
        \bm{w}_k \\
        \bm{v}_k
    \end{bmatrix}.
\]
The two recursions above can be written compactly as in \cref{eq:projAnomDynamics}:
\[
    \bm{z}_{k+1} = \bm{F}\bm{z}_k + \bm{G}\bm{\eta}_k,
\]
with $\bm{F}$ and $\bm{G}$  as  in \cref{eq:covMatF,eq:covMatG}:
\[
    \bm{F} =
    \begin{bmatrix}
        \bm{A}-\KalmanGain\bm{C} & \bm{0} \\
        \KalmanGain\bm{C} & \bm{A}
    \end{bmatrix},
    \qquad
    \bm{G} =
    \begin{bmatrix}
        \bm{I} & -\KalmanGain \\
        \bm{0} & \KalmanGain
    \end{bmatrix}.
\]

\subsection*{Covariance Propagation}
Let $\bm{P}_{z,k} \coloneqq \mathbb{E}[\bm{z}_k\bm{z}_k^\top]$. Taking the outer product of $\bm{z}_{k+1}$ as in \cref{eq:projAnomDynamics} yields
\begin{align*}
    \bm{P}_{z,k+1}
    &= \mathbb{E}\!\left[
        (\bm{F}\bm{z}_k + \bm{G}\bm{\eta}_k)
        (\bm{F}\bm{z}_k + \bm{G}\bm{\eta}_k)^\top
    \right] \\
    &= \bm{F}\,\mathbb{E}[\bm{z}_k\bm{z}_k^\top]\bm{F}^\top
    + \bm{F}\,\mathbb{E}[\bm{z}_k\bm{\eta}_k^\top]\bm{G}^\top \\
    &\quad
    + \bm{G}\,\mathbb{E}[\bm{\eta}_k\bm{z}_k^\top]\bm{F}^\top
    + \bm{G}\,\mathbb{E}[\bm{\eta}_k\bm{\eta}_k^\top]\bm{G}^\top.
\end{align*}
Under the standard Kalman-filter assumptions, the noise vector $\bm{\eta}_k$ is white, zero-mean, and independent of $\bm{z}_k$, which depends only on noises up to time $k-1$. Therefore,
\[
    \mathbb{E}[\bm{z}_k\bm{\eta}_k^\top] = \bm{0},
    \qquad
    \mathbb{E}[\bm{\eta}_k\bm{z}_k^\top] = \bm{0}.
\]
Moreover, because $\bm{w}_k$ and $\bm{v}_k$ are mutually independent with covariances $\bm{Q}$ and $\bm{R}$,
\[
    \mathbb{E}[\bm{\eta}_k\bm{\eta}_k^\top]
    =
    \begin{bmatrix}
        \bm{Q} & \bm{0} \\
        \bm{0} & \bm{R}
    \end{bmatrix}
    =
    \operatorname{diag}(\bm{Q},\bm{R}).
\]
Substituting these identities gives
\[
    \bm{P}_{z,k+1}
    =
    \bm{F}\bm{P}_{z,k}\bm{F}^\top
    + \bm{G}\operatorname{diag}(\bm{Q},\bm{R})\bm{G}^\top.
\]
With $\bm{\Pi}$ defined as in \cref{eq:covMatPi}, this is the recursion stated in \cref{eq:covRecurs}:
\[
    \bm{P}_{z,k+1} = \bm{F}\bm{P}_{z,k}\bm{F}^\top + \bm{\Pi}.
\]

\subsection*{Initial Condition and Extracted Covariance}
At the synchronization instant, \cref{eq:resync} implies
\[
    \tilde{\bm{r}}_0
    =
    \hat{\bm{x}}_0 - \ActuationProjectedState_0
    =
    \bm{0}.
\]
By Assumption~\ref{ass:converged}, the Kalman filter is already at steady state, so
\[
    \operatorname{Cov}(\bm{e}_0) = \bm{P}.
\]
Since $\tilde{\bm{r}}_0$ is deterministic (i.e., $\tilde{\bm{r}}_0 = \bm{0}$ from \cref{eq:resync} ), its covariance and cross-covariance with $\bm{e}_0$ vanish. Consequently,
\[
    \bm{P}_{z,0}
    =
    \begin{bmatrix}
        \operatorname{Cov}(\bm{e}_0) & \mathbb{E}[\bm{e}_0\tilde{\bm{r}}_0^\top] \\
        \mathbb{E}[\tilde{\bm{r}}_0\bm{e}_0^\top] & \operatorname{Cov}(\tilde{\bm{r}}_0)
    \end{bmatrix}
    =
    \begin{bmatrix}
        \bm{P} & \bm{0} \\
        \bm{0} & \bm{0}
    \end{bmatrix}
    =
    \operatorname{diag}(\bm{P},\bm{0}).
\]
Iterating \cref{eq:covRecurs} from this initial condition yields $\bm{P}_{z,k}$, and the desired covariance $\bm{\Sigma}_{\tilde{r},k} = \operatorname{Cov}(\tilde{\bm{r}}_k)$ is the lower-right block of $\bm{P}_{z,k}$, as claimed in \cref{thm:residual_cov}.

  \section{Proof of stability}\label{app:stability}
  \subsection*{Detailed Proof of Frozen-Gain Stability}
This appendix expands item~2 of \cref{th:main}.
By inverse-dynamics compensation, the joint dynamics reduce to the decoupled double-integrator model in \cref{eq:idealDynamics}. Therefore, it suffices to analyze a single joint (or DOF) $j$ with state $\bm{x}_{k,j}\in\mathbb{R}^2$ and scalar acceleration input $u_{k,j}\in\mathbb{R}$. Its sampled dynamics are
\begin{equation}
    \bm{x}_{k+1,j}=\bm{A}_{j}\bm{x}_{k,j}+\bm{B}_{j}u_{k,j},
    \bm{A}_{j}=\begin{bmatrix}1&T_s\\0&1\end{bmatrix},\;
    \bm{B}_{j}=\begin{bmatrix}T_s^2/2\\T_s\end{bmatrix},
\end{equation}
where $T_s > 0$ is the sampling time.

Let the nominal state feedback induced by the controller gains of \cref{eq:pid_actuation_task,eq:pid_actuation_joint} be
\begin{equation}
    u_{k,j}^{\text{nom}}=-\bm{K}\bm{x}_{k,j},\qquad \bm{K}=[K_p\;K_d],
\end{equation}
and assume that the corresponding nominal closed-loop matrix
\[
    \bm{A}_{\mathrm{cl}}(1)=\bm{A}_{j}-\bm{B}_{j}\bm{K}
\]
is Schur stable. For discrete-time linear systems, Schur stability is equivalent to asymptotic stability.

Freezing the command scaling to a constant $\bar f\in(0,1]$ yields
\begin{equation}
    u_{k,j} = \bar f\, u_{k,j}^{\text{nom}} = -\bar f\,\bm{K}\bm{x}_{k,j},
\end{equation}
which is the frozen-gain case of \cref{eq:finalControl}. Substituting this expression into the sampled double-integrator dynamics gives
\[
\bm{x}_{k+1,j}=\bm{A}_{j}\bm{x}_{k,j}+\bm{B}_{j}\bigl(-\bar f\,\bm{K}\bm{x}_{k,j}\bigr)
=\bigl(\bm{A}_{j}-\bar f\,\bm{B}_{j}\bm{K}\bigr)\bm{x}_{k,j},
\]
and thus the frozen-gain closed-loop matrix is
\begin{equation}
    \bm{A}_{\mathrm{cl}}(\bar f)=\bm{A}_{j}-\bar f\,\bm{B}_{j}\bm{K} =
\begin{bmatrix}
	1-\tfrac{\bar f T_s^2}{2}K_p & T_s-\tfrac{\bar f T_s^2}{2}K_d \\
	-\bar f T_s K_p & 1-\bar f T_s K_d
\end{bmatrix}.
\end{equation}

The characteristic polynomial of $\bm{A}_{\mathrm{cl}}(\bar f)$ is
\begin{equation}
    \lambda^2-\Big(2-\bar f T_sK_d-\tfrac{\bar f T_s^2}{2}K_p\Big)\lambda+\Big(1-\bar f T_sK_d+\tfrac{\bar f T_s^2}{2}K_p\Big).
\end{equation}
For a second-order polynomial $\lambda^2+a_1\lambda+a_0$, the Jury criterion states that Schur stability is equivalent to
\[
    1+a_1+a_0>0,\qquad 1-a_1+a_0>0,\qquad 1-a_0>0.
\]
Applying these conditions to the characteristic polynomial above, with
\[
    a_1 = -\Bigl(2-\bar f T_sK_d-\tfrac{\bar f T_s^2}{2}K_p\Bigr),
    \qquad
    a_0 = 1-\bar f T_sK_d+\tfrac{\bar f T_s^2}{2}K_p,
\]
gives
\begin{align*}
    1+a_1+a_0 &= \bar f T_s^2K_p, \\
    1-a_1+a_0 &= 2\bigl(2-\bar f T_sK_d\bigr), \\
    1-a_0 &= \bar f T_s\Bigl(K_d-\tfrac{T_s}{2}K_p\Bigr).
\end{align*}
Hence, $\bm{A}_{\mathrm{cl}}(\bar f)$ is Schur stable if and only if
\begin{equation}
    \bar f T_s^2K_p>0,\qquad 2-\bar f T_sK_d>0,\qquad K_d>\tfrac{T_s}{2}K_p.
\end{equation}
Since the nominal matrix $\bm{A}_{\mathrm{cl}}(1)$ is Schur stable by assumption, these Jury conditions hold at $\bar f=1$. Therefore,
\[
    K_p>0,\qquad 2-T_sK_d>0,\qquad K_d>\tfrac{T_s}{2}K_p.
\]
Now consider any $\bar f\in(0,1]$. The first and third inequalities remain valid because $\bar f>0$, while the second becomes less restrictive:
\[
    2-\bar f T_sK_d \ge 2-T_sK_d > 0.
\]
Thus all Jury conditions continue to hold for every frozen scaling factor $\bar f\in(0,1]$. Consequently, $\bm{A}_{\mathrm{cl}}(\bar f)$ is Schur stable for all such $\bar f$, which proves the frozen-gain stability claim in \cref{th:main}.

For a manipulator with $p$ decoupled double integrators, the frozen-gain closed-loop matrix is block diagonal across DOFs. The single-DOF argument above applies componentwise.


\end{document}